%% file: main.tex
\documentclass[journal]{IEEEtran}
\usepackage{amsmath,amssymb,amsfonts}
\usepackage{algorithmic}
\usepackage{xcolor}
\usepackage{cite}
\usepackage{graphicx}
\usepackage{url} 
\usepackage{multirow}
\usepackage{booktabs}

\usepackage{algorithm}
\usepackage{subfigure}

\newcommand{\tensor}[1]{\uppercase{\mathcal{#1}}}
\newcommand{\matrics}[1]{\uppercase{\mathbf{#1}}}
\newcommand{\vectors}[1]{\lowercase{\mathbf{#1}}}
\newcommand{\scalars}[1]{\lowercase{#1}}

\newcommand{\graset}[1]{\uppercase{\emph{#1}}}
\newcommand{\grasubset}[1]{\uppercase{\emph{#1}}}
\newcommand{\gravex}[1]{\lowercase{#1}}

\newcommand{\etal}{\emph{et~al. }}

\newcommand{\EN}{{$\epsilon_{N}$}}
\newcommand{\EE}{{$\epsilon_{E}$}}
\newcommand{\ET}{{$\epsilon_{T}$}}
\newcommand{\TOOLNAME}{{MOSC}}

\newcommand{\ctwo}{{\phi_2}}
\newcommand{\ntwo}{{\eta_2}}

\newcommand{\nth}{{\eta_3}}
\newcommand{\nassoc}{{\xi_3}}
\newcommand{\nassocTwo}{{\xi_2}}

\newcommand{\ccut}{{\phi}}
\newcommand{\ncut}{{\eta}}
\newcommand{\cncut}{{\xi}}

\usepackage{amsthm}
\newtheorem{theorem}{Theorem}%

\newtheorem{lemma}{Lemma}

\ifCLASSINFOpdf
\else
\fi

\begin{document}
\title{Mixed-Order Spectral Clustering for Networks}

\author{Yan Ge, Haiping Lu, and Pan Peng
\thanks{Yan Ge, Haiping Lu and Pan Peng are with the Department of Computer Science, University of Sheffield, U.K. (e-mail: yge5@sheffield.ac.uk, h.lu@sheffield.ac.uk and p.peng@sheffield.ac.uk). }}

\markboth{Preprint}%
{Shell \MakeLowercase{\textit{et al.}}: Bare Demo of IEEEtran.cls for IEEE Journals}

\maketitle

\input{Abstract}

\input{Introduction}

\input{Preliminary}

\input{Methodology}

\input{Experiments}

\input{Conclusion}
\input{Appendix}

\IEEEpeerreviewmaketitle
\bibliographystyle{IEEEtran}
\bibliography{references}

\end{document}

%% file: Abstract.tex
\begin{abstract}
Clustering is fundamental for gaining insights from complex networks, and spectral clustering (SC) is a popular approach. Conventional SC focuses on \textit{second-order} structures (e.g., edges connecting two nodes) without direct consideration of \textit{higher-order} structures (e.g., triangles and cliques). This has motivated SC extensions that directly consider higher-order structures. However, both approaches are limited to considering a single order. This paper proposes a new \textit{Mixed-Order} Spectral Clustering (MOSC) approach to model both second-order and third-order structures simultaneously, with two MOSC methods developed based on Graph Laplacian (GL) and Random Walks (RW). MOSC-GL combines edge and triangle adjacency matrices, with theoretical performance guarantee. MOSC-RW combines first-order and second-order random walks for a probabilistic interpretation. We automatically determine the mixing parameter based on cut criteria or triangle density, and construct new structure-aware error metrics for performance evaluation. Experiments on real-world networks show 1) the superior performance of two MOSC methods over existing SC methods, 2) the effectiveness of the mixing parameter determination strategy, and 3) insights offered by the structure-aware error metrics.

\vspace{2mm}
\end{abstract}

%% file: Introduction.tex
\section{Introduction}
\IEEEPARstart{N}{etworks}
 (a.k.a. graphs) are important data structures that abstract relations between discrete objects, such as social networks and brain networks \cite{benson2016higher}. A network is composed of nodes and edges representing node interactions. \emph{Clustering} is an important and powerful tool in analysing network data, e.g., for community detection \cite{liu2014multiobjective,yang2015unified,chang2017community,zhang2017mixed,li2018multiobjective}. 

Clustering aims to
divide the data set into \emph{clusters} (or \emph{communities}) such that the nodes assigned to a particular cluster are similar or well connected in some predefined sense \cite{clauset2004finding,schaeffer2007graph,POM09:communities,fortunato2010community,New12:communities}. 
It helps us reveal functional groups hidden in data.
As a popular clustering method, conventional spectral clustering (SC)~\cite{meila2001learning, ng2002spectral} encodes 
pairwise similarity into an adjacency matrix.
Such encoding inherently restricts SC to \textit{second-order structures} \cite{benson2016higher}, such as undirected or directed edges connecting two nodes.\footnote{Edges are considered as first-order structures in \cite{benson2015tensor} but second-order structures in \cite{zhou2017local}. We follow the terminologies in the latter \cite{zhou2017local} so that the ``order'' here refers to the number of nodes involved in a particular structure.} 
However, in many real-world networks, the minimal and functional structural unit of a network is not a simple edge but a small network subgraph (a.k.a. \emph{motif}) that involves more than two nodes \cite{milo2002network}, which we call a higher-order structure. %

\textit{Higher-order structures} consist of at least three nodes (e.g., triangles, $4$-vertex cliques) \cite{benson2016higher}. 
It can directly capture interaction among three or more nodes.
When clustering networks, higher-order structures can be regarded as fundamental \emph{units} and algorithms can be designed to minimise cutting them in partitioning.
Clustering based on higher-order structures can help us gain new insights and significantly improve our understanding of underlying networks.
For example, triangular structures, with three reciprocated edges connecting three nodes, play important roles in brain networks \cite{sporns2004motifs} and social networks \cite{granovetter1977strength, kossinets2006empirical}. 
More importantly, higher-order structures allow for more flexible modelling. For instance, considering directions of edges, there exist
$13$ different {third-order structures},
but only two different second-order structures \cite{serrour2011detecting}. Thus, the application can drive which
third-order structures to be preserved. 

Thus, there are emerging interests in directly modelling higher-order structures in network clustering. These works can be grouped into four approaches: 1) The first approach constructs an affinity tensor to encode higher-order structures and then reduces it to a matrix \cite{ghoshdastidar2015spectral,ghoshdastidar2014consistency}, followed by conventional SC \cite{ng2002spectral}. These methods, such as tensor trace maximisation (TTM) \cite{ghoshdastidar2017uniform}, are developed in a closely related problem of hypergraph clustering that considers ``hyperedges'' connecting multiple nodes. 2) The second approach develops higher-order SC by constructing a transition tensor based on random walks model and then reduces it to a matrix for conventional SC, such as tensor spectral clustering (TSC) \cite{benson2015tensor, wu2016general}. 
3) The third approach uses a counting and reweighting scheme to capture higher-order structures and reveal clusters \cite{tsourakakis2017scalable}, such as motif-based SC (MSC) \cite{benson2016higher}.\footnote{We have verified that TTM and MSC are equivalent, nevertheless.} 
4) The fourth approach is higher-order local clustering aiming to  reduce computation cost \cite{yin2017local}, such as  High-Order Structure Preserving Local Clustering (HOSPLOC) \cite{zhou2017local}.

\begin{figure} [!t]
	\centering
\subfigure[A network.]{\includegraphics[width=0.13\textwidth]{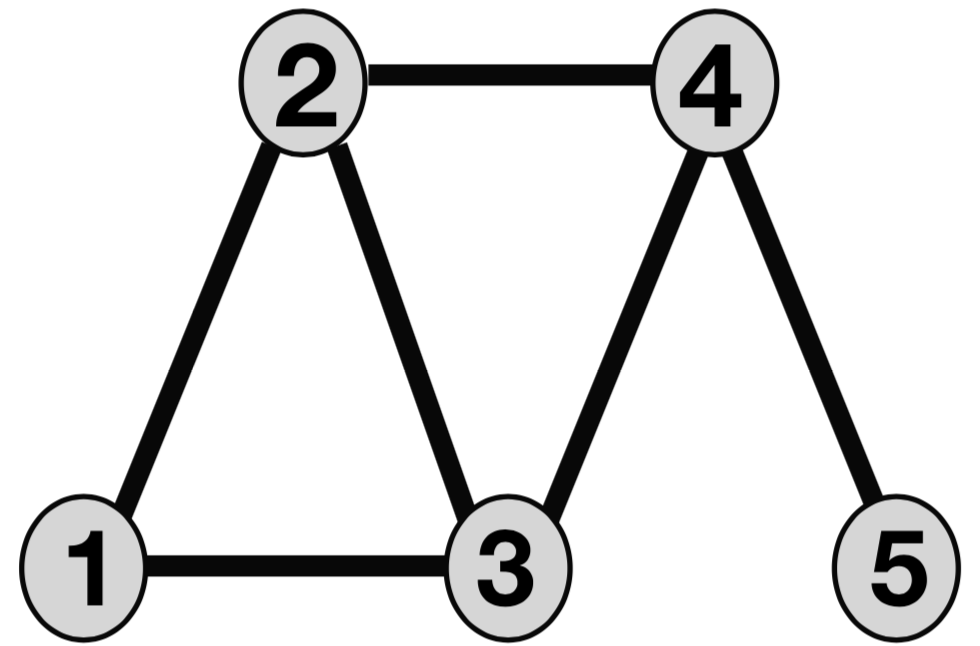}\label{demo_graph}}
\subfigure[Edges.]{\includegraphics[width=19.5mm,height=19.5mm]{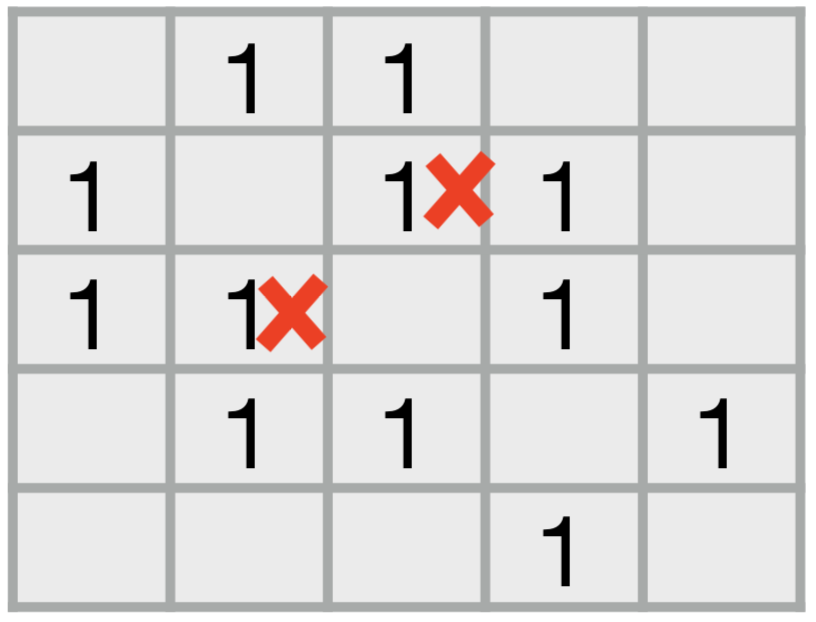}\label{demo_edge}}
\subfigure[Triangles.]{\includegraphics[width=19.5mm,height=19.5mm]{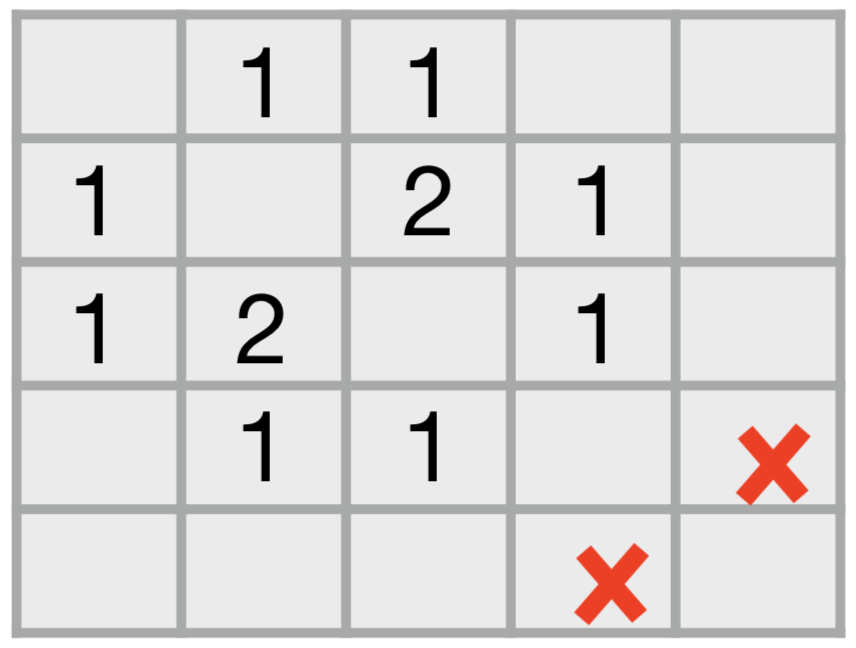}\label{demo_motif}}
\subfigure[Mixed.]{\includegraphics[width=19.5mm,height=19.5mm]{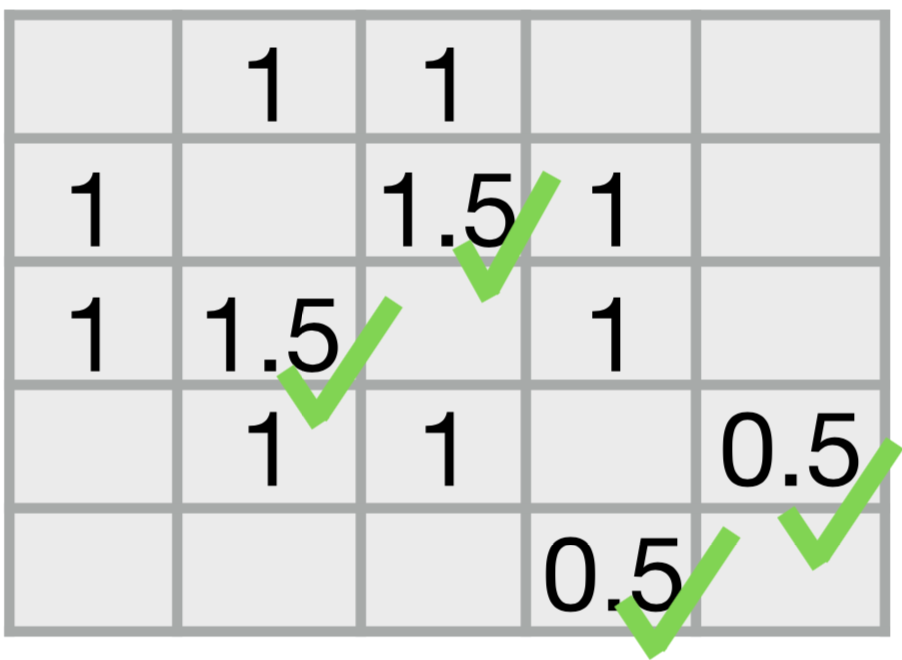}\label{demo_mix}}\vspace{-1mm}
	\caption{The second and third order structures in (a) can not be fully captured by edge/triangle adjacency matrix in (b)/(c). A mixed adjacency matrix in (d) can capture both.}
	\label{demo}\vspace{-6mm}
\end{figure}

However, it should be noted that most networks have both second-order and higher-order structures, e.g., the network in Fig. \ref{demo_graph}, and both can be important. Existing conventional and third-order SC methods model only either second-order or third-order structures, but \emph{not both}, as shown in Figs. \ref{demo_edge} and \ref{demo_motif}, which fail to properly model the importance between nodes 2 and 3, and nodes 4 and 5, respectively. In Figs. \ref{demo_edge} and \ref{demo_motif}, each entry indicates the number of edges and triangles involving two nodes of Fig.~\ref{demo_graph}.

In this paper, we propose a new \textit{Mixed-Order Spectral Clustering} (MOSC) to preserve structures of different orders simultaneously, as in Fig. \ref{demo_mix}. For clear and compact presentation, we focus on two \emph{undirected} \textit{unweighted} structures: edges (second-order structures) and triangles (third-order structures). 
Further extensions can be developed for mixing more than two orders, and/or orders higher than three.

We summarise our four contributions as following:
\begin{enumerate}
\item \textbf{Mixed-order models}. 
We develop two MOSC models: one based on Graph Laplacian (GL) and the other based on Random Walks (RW). MOSC-GL combines edge and triangle adjacency matrices to define a mixed-order Laplacian, with its theoretical performance guarantee derived by proving a mixed-order Cheeger inequality. MOSC-RW combines first-order and second-order RW models for a probabilistic interpretation. The final clusters are obtained via a sweep cut procedure or $k$-means. See Sec. \ref{MOSC-GL} and Sec. \ref{MOSC-RW}. 
\item \textbf{Automatic model selection.}
There is only one hyperparameter in MOSC, i.e., the mixing parameter (ranging from 0 to 1). We propose cut-criteria-based and triangle-density-based strategies to automatically determine the best mixing parameter. See Sec. \ref{autoLambda}. 
\item \textbf{Structure-aware error metrics}. Existing works on higher-order structure clustering use evaluation metrics focused on \emph{mis-clustered nodes} \cite{tsourakakis2017scalable,yin2017local,ghoshdastidar2017uniform}. 
However, mis-clustered nodes do not have a monotonic relationship with mis-clustered structures so they may fail to reflect the errors in structures. Therefore, we propose  structure-aware error metrics to gain insights on the quality of structure preservation. See Sec. \ref{secENode}. %
\item \textbf{Cut criterion}. A cut criterion evaluates the quality of output clusters.
Therefore, we study seven different sweep cut criteria (both second and third orders) and propose a new one by defining mixed-order conductance. We also propose an \textit{optimal cut} to study the lower bound of error metrics and model quality. See Secs.~\ref{cut_cmp} and \ref{sec:performance}.
\end{enumerate}

%% file: Preliminary.tex
\section{Preliminaries}\label{sec:preliminaries}
\subsection{Notations}
We denote scalars by lowercase letters, e.g., $\scalars{a}$, 
vectors by lowercase boldface letters, e.g., $\vectors{a}$, matrices by  uppercase  boldface, e.g., $\mathbf{A}$, and tensors by calligraphic letters, e.g., $\tensor{A}$. 
Let $\graset{G} = (\graset{V}, \graset{E})$ be an undirected unweighted graph (network) with $\graset{V}$ = \{\gravex{v$_1$, v$_2$, \dots, v$_n$}\} being the set of $n$ vertices (nodes), i.e., $\scalars{n}$ = $|\graset{V}|$, and \graset{E} being the set of edges connecting two vertices.
\subsection{Normalised Graph Laplacian}
Let $\mathbf{W}\in\mathbb{ R}^{n\times n}$ 
be an unweighted adjacency matrix of $\graset{G}$  
where $\mathbf{W}(i,j)$ = 1 if (\gravex{v$_i$}, \gravex{v$_j$}) 
$\in \graset{E}$, otherwise $\mathbf{W}(i,j)$ = 0.
The degree matrix $\mathbf{D}$ is a diagonal matrix with diagonal entries $\mathbf{D}(i,i) = \sum_{j=1}^n\mathbf{W}(i,j)$, which is the \textit{degree} of vertex v$_i$.
Let $\mathbf{N}=\mathbf{D}-\mathbf{W}$ denote the \emph{Laplacian matrix} of $G$. The \emph{normalised Laplacian} of $G$ is defined as:
\begin{equation}
  \mathbf{{L}} =
\mathbf{D}^{-\frac{1}{2}}\mathbf{N}\mathbf{D}^{-\frac{1}{2}}.   
\end{equation}

Let $\mathbf{W}_T$ be \emph{triangle adjacency matrix}  of $\graset{G}$ with its entry ($i$, $j$) being the number of triangles containing vertices $i$ and $j$, which leads to a  corresponding weighted graph $G_T$  \cite{benson2016higher}. %
Similarly, we can define the \emph{triangle Laplacian} as $\mathbf{N}_T=\mathbf{D}_T - \mathbf{W}_T$ and the \textit{normalised triangle Laplacian} as:
\begin{equation}
 \mathbf{{L}}_T = \mathbf{D}_T^{-\frac{1}{2}}\mathbf{N}_T\mathbf{D}_T^{-\frac{1}{2}},   
\end{equation}
 where $\mathbf{D}_T(i,i) = \sum^n_{j = 1} \mathbf{W}_T(i,j)$.

\subsection{First-Order Random Walks for Second-Order Structures}
We define a second-order transition matrix $\mathbf{P}$ by normalising the adjacency matrix $\mathbf{W}$ to represent edge structures as~\cite{meila2001learning}:
\begin{equation}
\mathbf{P} = \mathbf{D}^{-1} \mathbf{W}\label{trans_matrix}.
\end{equation}
The entry $\mathbf{P}_{ij}$ represents the probability of jumping from vertex \gravex{v$_i$} to \gravex{v$_j$} in one step.
The transition matrix $\mathbf{P}$
represents a random walk process on graph $\graset{G}$\cite{meila2001learning}. 
From random walk perspective, SC can be interpreted as trying to find
a partition of the graph such that the random walk stays long within the same cluster and
seldom jumps between clusters~\cite{von2007tutorial}. 

\subsection{Second-Order Random Walks for Third-Order Structures}
Benson \etal \cite{benson2015tensor} extend the above using a three-dimensional transition tensor to encode triangle structures. 
They firstly define a symmetric \textit{adjacency tensor}  $\tensor{T}\in\mathbb{ R}^{n\times n\times n}$
such that the connectivity information for three vertices  
\{\gravex{v$_i$}, \gravex{v$_j$}, \gravex{v$_k$}\}$\in \graset{V}$ can be represented explicitly in this tensor. All entries in  $\tensor{T}$ with a permutation
of indices $\scalars{i}$, $\scalars{j}$, $\scalars{k}$  have the same value (hence symmetric). Thus, $\tensor{T}$ encodes triangle structures 
in $\graset{G}$ as: 
\begin{equation}\label{adj-tensor}\tensor{T} (i, j, k)=\left\{
\begin{aligned}
1 &&  \textrm{$v_i, v_j, v_k$  form a triangle,}\\
0 && \textrm{otherwise.}
\end{aligned}
\right.\end{equation}
Next, they form a third-order transition tensor $\tensor{P}$ as: 
\begin{equation} \label{trans-tensor}
\tensor{P} (i, j, k) =
\tensor{T}(i, j, k) / \sum_{m=1}^n \tensor{T}(i, m, k),
\end{equation}
where $\sum_{m=1}^n \tensor{T}(i, m, k) \neq$ 0, and  1 $\leq$ $\scalars{i}$, $\scalars{j}$, $\scalars{k}$
$\leq$ $\scalars{n}$. For $\sum_{m=1}^n \tensor{T}(i, m, k) = $ 0, 
Benson \etal \cite{benson2015tensor} set $\tensor{P} (i, j, k)$ to $\frac{1}{n}$.

Here, the entries of $\tensor{P}$ represent the transition probabilities of the second-order random walks.
Each vertex in \graset{V} is considered a distinguishable state.
And probability of jumping to state $\scalars{i}$ relies on the last two states $\scalars{j}$ and $\scalars{k}$\cite{xu2016representing}.
This probabilistic interpretation implies that random walks uniformly choose state $\scalars{i}$ that forms a triangle with 
$\scalars{j}$ and $\scalars{k}$. However, analysing $\tensor{P}$ is challenging, e.g., finding its eigenvectors is NP-hard \cite{malliaros2013clustering}, so it is often reduced to a similarity matrix to apply conventional SC procedures \cite{benson2015tensor}.

\subsection{Spectral Clustering Basics}
Bi-partitioning SC (Algorithm \ref{spectral_algs}) first constructs a  matrix $\mathbf{B}$ to encode structures in the input graph $\graset{G}$. It then computes a dominant eigenvector $\vectors{v}$ of $\mathbf{B}$, thus making use of its \textit{spectrum}. Each entry of $\vectors{v}$ corresponds to a vertex. Next, we sort vertices by the values $\vectors{v}(i)$ (or appropriately normalised values) and consider the set $T_u$ consisting of the first $u$ vertices in the sorted list, for each $1\leq u\leq n-1$. Then the algorithm finds $S=\arg\min_{T_u}\tau(T_u)$, called the \emph{sweep cut} w.r.t. some cut criterion $\tau$. Table \ref{table:cmp_cut} lists eight cut criteria, both second and third orders (edge and triangle based). We omit $G$ in the criterion notation $\tau(S;G)$ in the table.

Matrix $\mathbf{B}$ is a crucial part of SC. Conventional (i.e., second-order) SC starts with the \emph{adjacency matrix} $\mathbf{W}$ of $\graset{G}$ encoding edges (as in Fig. \ref{demo_edge}) and defines $\mathbf{B}$ as $\mathbf{W}$, the Laplacian matrix, or their normalised versions~\cite{schaeffer2007graph}. Third-order SC typically starts with a third-order tensor $\tensor{T}$ encoding triangles, which can be reduced to a matrix $\mathbf{B}$ (as in Fig. \ref{demo_motif}) to apply Algorithm \ref{spectral_algs}.

\begin{algorithm}[!t]
	\caption{Bi-partitioning Spectral Clustering}%
	\begin{algorithmic} [1]\label{spectral_algs}
		\STATE Matrix $\mathbf{B}$ encodes structures of the input graph $\graset{G}$.
		\STATE Compute a dominant eigenvector $\vectors{v}$ of $\mathbf{B}$.
		\STATE $\vectors{v}  \gets$ sorted ordering of $\vectors{v}$ or a normalised version of $\vectors{v}$.
		\STATE \{$S$, $\bar{S}$\} $\gets$ sweep cut of $\vectors{v}$ w.r.t. some cut criteria. 
	\end{algorithmic}
\end{algorithm}\vspace{-5mm}

\begin{small}
\begin{center}
	\begin{table*}[!t]
		\caption{Edge-based and triangle-based cut criteria.}\label{table:cmp_cut}
		\def\arraystretch{1.59}
		\resizebox{\textwidth}{!}{
			\begin{tabular}{ l c c }
				\toprule
				& \textbf{Edge-based cuts} & \textbf{Triangle-based cuts} \\ 
				\midrule
				Conductance ($\ccut$) & 
				\multicolumn{1}{l}{$\ctwo$($S$) = 
					$\frac{cut_2(S)}
					{\min(vol_2(S),vol_2(\bar{S}))}\label{conductance}$ \cite{schaeffer2007graph}}
				& 
				\multicolumn{1}{l}{$\phi_3$($S$) =
					$\frac{cut_3(S)}
					{\min(vol_3(S),vol_3(\bar{S}))}\label{3-conductance}$ \cite{benson2015tensor}}\\ 
				Ncut ($\ncut$)& $\ntwo(S) = $$
				cut_2(S)(\frac{1}{{vol_2(S)}}+
				\frac{1}{vol_2({\bar{S})}}) \label{2-Ncut}$ \cite{shi2000normalized} 
				
				& $\nth(S)$ = $
				cut_3(S)(\frac{1}{{vol_3({S})}}+
				\frac{1}{vol_3(\bar{S})})\label{3-Ncut}$ \cite{li2017inhomogoenous} \\ 
				Nassoc ($\cncut$)& \multicolumn{1}{l}{$\nassocTwo({S})$ =
					$\frac{assoc_2(S)}{vol_2(S)}+
					\frac{assoc_2(\bar{S})}{vol_2(\bar{S})}$\label{second-Nassoc}\cite{shi2000normalized}} 
				& 
				\multicolumn{1}{l}{ $\nassoc({S}) =
					$$\frac{assoc_3(S)}{vol_3(S)}+
					\frac{assoc_3(\bar{S})}{vol_3(\bar{S})} \label{third-Nassoc}$\cite{ghoshdastidar2017uniform}} \\ 
				Expansion ($\alpha$) & 
				\multicolumn{1}{l}{$\alpha_2(S)$ =
					$\frac{cut_2(S)}
					{\min(|S|,|\bar{S}|)}\label{expen2}$ \cite{flake2004graph}}
				& 
				\multicolumn{1}{l}{$\alpha_3(S)$ = $
					\frac{cut_3(S)}
					{\min(|S|,|\bar{S}|)}\label{expen3}$ \cite{benson2015tensor}}\\ 
				\bottomrule
			\end{tabular}
		}
	\end{table*}
\end{center}
\end{small}
\subsection{Cheeger Inequalities and Cut Criteria} Given $\graset{G} = (\graset{V}, \graset{E})$ and a subset $\graset{S}\subseteq
\graset{V}$, let $\bar{S}$ denote the complement of $\graset{S}$. Let $cut_2(S;G)$
denote the \emph{edge cut} of $S$, i.e, the number of edges between $S$ and $\bar{S}$ in $G$. 
Let $vol_2(S;G)$ denote the \emph{edge volume} of $S$, i.e., the total degrees of
vertices in $S$, and $assoc_2(S;G)$ is the total degrees in the subgrapgh induced by vertices in $S$. The \emph{edge conductance} of $S$ is defined as

\begin{equation}
\phi_2(S;G)=\frac{cut_2(S;G)}{\min\{vol_2(S;G),vol_2(\bar{S};G)\}}.
\end{equation}
Other popular edge-based cut criteria are shown in Table \ref{table:cmp_cut} (left column).
The classical Cheeger inequality below relates the conductance of the sweep cut of SC to the minimum conductance value of the graph \cite{fiedler1973algebraic}. 
\begin{lemma}[Second-Order Cheeger Inequality]\label{lem:secondcheeger}
Let $\vectors{v}$ be the second smallest eigenvector of $\mathbf{L}$. Let $T^*$ be the sweep cut of $\mathbf{D}^{-1/2}\vectors{v}$ w.r.t. cut criterion $\phi_2(\cdot;G)$. It holds that  %
\begin{equation}
    \phi_2(T^*;G)\leq2\sqrt{\phi_2^*},
\end{equation}
 where $\phi_2^* = \min_{\grasubset{S}\subset\graset{V}} \  \phi_2(\grasubset{S};G)$
is the minimum conductance over any set of vertices. 
\end{lemma}

Let $cut_3(S;G)$ denote the \emph{triangle cut} of $S$, i.e., the number of triangles that have at least one endpoint in $S$ and at least one endpoint in $\bar{S}$, and let $assoc_3(S;G)$ count the number of vertices in triangles in the subgraph induced by vertices in $S$. Let $vol_3(S;G)$ denote the \emph{triangle volume} of $S$, i.e., the number of triangle endpoints in $S$.
The \emph{triangle conductance}~\cite{benson2015tensor} of $S$ is defined as
\begin{equation}
    \phi_3(S;G)=\frac{cut_3(S;G)}{\min\{vol_3(S;G),vol_3(\bar{S};G)\}}.
\end{equation}
  It is further proved in \cite{benson2016higher} that for any $S\subset V$, $\phi_3(S;G)=\phi_{2}(S;G_T)$, which leads to the following third-order Cheeger inequality.  Other popular triangle-based cut criteria are summarised in Table \ref{table:cmp_cut} (right column).

\begin{lemma}[Third-order Cheeger Inequality] \label{lem:thirdcheeger}
Let $\vectors{v}$ be the second smallest eigenvector of $\mathbf{L}_T$. Let $T^*$ denote the sweep cut of $\mathbf{D}^{-1/2}_T\vectors{v}$ w.r.t. cut criteria $\phi_{2}(\cdot;G_T)$. It holds that  
\begin{equation}
\phi_3(T^*;G) \leq 4 \sqrt{\phi_3^*},
\end{equation}
 where $\phi_3^* = min_{\grasubset{S}\subset\graset{V}} \  \phi_3(\grasubset{S};G)$.
\end{lemma}

%% file: Methodology.tex
\section{Proposed Mixed-Order Approach}
To model both edge and triangle structures simultaneously, we introduce a new Mixed-Order SC (MOSC) approach, with two methods based on Graph Laplacian (GL) and Random Walks (RW). MOSC-GL combines the edge and triangle adjacency matrices, %
which leads to a mixed-order Cheeger inequality to provide a theoretical performance guarantee. MOSC-RW is developed under the random walks framework to combine the first and second order random walks, providing a probabilistic interpretation. 
Next, we develop an automatic hyperparameter determination scheme and define new structure-aware error metrics. Finally, we present a way to examine the lower bound of the clustering errors. %

\subsection{MOSC Based on Graph Laplacian (MOSC-GL)}\label{MOSC-GL}

MOSC-GL introduces a \textit{mixed-order adjacency matrix} $\matrics{W}_{X}$ that linearly combines the edge adjacency matrix $\matrics{W}$ and the triangle adjacency matrix $\matrics{W}_T$,
with a \textit{mixing parameter} $\lambda\in [0,1]$.  $\matrics{W}_{X}$ can be seen as a weighted adjacency matrix of a weighted graph $G_{X}$, on which we can apply conventional SC (Algorithm \ref{spectral_algs}). Specifically, we first construct the matrix $\matrics{W}_{X}$ and the corresponding diagonal degree matrix $\matrics{D}_{X}$ as
\begin{eqnarray}
\matrics{W}_{X} &=& (1-\lambda)\matrics{W}_T+\lambda\matrics{W}, \label{mix-adj} \\
\matrics{D}_{X} &=& (1-\lambda)\matrics{D}_T+\lambda \matrics{D}. \label{Dmix}
\end{eqnarray}

Let $G_{X}$ denote an undirected weighted graph with adjacency matrix $\matrics{W}_{X}$, we can define a mixed-order Laplacian $\matrics{N}_{X}$ and its normalised  version $\matrics{L}_{X}$ as:
\begin{eqnarray}
\matrics{N}_{X} &=& \matrics{D}_{X} - \matrics{W}_{X}=(1-\lambda)\matrics{N}_T + \lambda\matrics{N}, \nonumber\\
\matrics{L}_{X} &=& \matrics{D}_{X}^{-\frac{1}{2}}\matrics{N}_X\matrics{D}_{X}^{-\frac{1}{2}}.%
\label{eqn:normalized_laplacian}
\end{eqnarray}
Then, we compute the eigenvector corresponding to the second smallest eigenvalue of  $\matrics{L}_{X}$ and perform the sweep cut to find the partition with the smallest edge conductance in $G_{X}$. The MOSC-GL algorithm is summarised in  Algorithm~\ref{alg:MOSC-GL}.

\begin{algorithm}[t]
	\caption{MOSC-GL}
	\begin{algorithmic} [1]\label{alg:MOSC-GL}
		\REQUIRE {$\graset{G}$ = ($\graset{V}$, $\graset{E}$)}, a mixing parameter $\lambda$
		\ENSURE Two node sets $\{S, \bar{S}\}$
		\STATE Construct the edge adjacency matrix $\matrics{W}\in 
		\mathbb{ R}^{n\times n}$.
		\STATE Construct the triangle adjacency matrix $\matrics{W}_T \in 
		\mathbb{ R}^{n\times n}$. 		 
		\STATE Let $\matrics{D}$ be diagonal with $\matrics{D}(i,i)$ = $\sum_i^n \matrics{W}(i,j)$.
		\STATE Let $\matrics{D}_T$ be diagonal with $\matrics{D}_{T}(i,i)$ = $\sum_i^n \matrics{W}_{T}(i,j)$.
		\STATE $\matrics{W}_{X} = (1-\lambda)\matrics{W}_T+\lambda\matrics{W}$.
		\STATE $\matrics{D}_{X} = (1-\lambda)\matrics{D}_T+\lambda \matrics{D}$.
		\STATE $\matrics{N}_{X} = \matrics{D}_{X} - \matrics{W}_{X}=(1-\lambda)\matrics{N}_T + \lambda\matrics{N}$.
		\STATE $\matrics{L}_{X} = \matrics{D}_{X}^{-\frac{1}{2}}\matrics{N}_X\matrics{D}_{X}^{-\frac{1}{2}}$.
		\STATE Compute the second smallest eigenvector $\vectors{v}_X$ of $\matrics{{L}}_{X}$.
		\STATE $\vectors{v}_X  \gets$ Sorting entries of $\matrics{D}_{X}^{-\frac{1}{2}}\vectors{v}_X$. 
		\STATE  \{$S$, $\bar{S}$\} $\gets$ Sweep cut on $\vectors{v}_X$ w.r.t. some cut criteria. 
	\end{algorithmic}
\end{algorithm}

When $\lambda=1$, MOSC-GL is equivalent to SC by Ng \emph{et~al.}~\cite{ng2002spectral} and only considers  second-order structures.
When $\lambda=0$, MOSC-GL is equivalent to motif-based SC \cite{benson2016higher}. MOSC-GL maintains the advantages of traditional SC: computational efficiency,
ease of implementation and mathematical guarantee on the near-optimality of resulting clusters, which we formalise and prove in the following. 

\textbf{Performance guarantee.}
Given a graph $G$ and a vertex set $S$, we define its mixed-order cut and volume as 
\begin{equation}
cut_{X} (S;G) = (1-\lambda)cut_3 (S;G) + \lambda cut_2 (S;G),
\end{equation}
and
\begin{equation}
vol_{X} (S;G) = (1-\lambda)vol_3 (S;G)+\lambda vol_2 (S;G),
\end{equation}
respectively. Then,  we define the \emph{mixed-order conductance} of $S$ as: 
\begin{equation}\label{mix-conductance}
\phi_{X} (S;G)= \frac{cut_{X} (S;G)}{\min(vol_{X} (S;G), vol_{X}(\bar{S};G))},
\end{equation}
which generalises edge and triangle conductance. A partition with small $\phi_{X} (S;G)$ corresponds to clusters with rich edge and triangle structures within the same cluster while few both structures crossing clusters. 
Finding the exact set of nodes $S$ with the smallest $\phi_{X}$ is computationally infeasible. Nevertheless, we can derive a performance guarantee for MOSC-GL to show that the output set obtained from Algorithm \ref{alg:MOSC-GL} is a good approximation.
To prove Theorem \ref{the_1}, we need the following Lemma.

\begin{lemma}[Lemma 4 and 1 in \cite{benson2016higher}]\label{lemma:bensor}
Let G  = (V, E) be an undirected, unweighted graph and let $G_T$ be 
the weighted graph for the triangle adjacency matrix. Then for any $S \subset V$, it holds that
\begin{eqnarray*}
cut_3(S;G)&=&\frac12 cut_2(S;G_T), \\
vol_3(S;G)&=&\frac12 vol_2(S;G_T).
\end{eqnarray*}
\end{lemma}
 
\begin{theorem}[Mixed-order Cheeger Inequality] \label{the_1}
Given an undirected graph $G$, let $T^*$ denote the set outputted by MOSC-GL (Algorithm \ref{alg:MOSC-GL}) w.r.t. the cut criterion $\phi_{2}(\cdot;G_X)$. Let $\phi_X^*=\min_{S\subseteq V} \phi_{X}(S;G)$ be the minimum mixed-order conductance over any set of vertices. Then it holds that
	$$
	\phi_{X}(T^*;G)  \leq 2\sqrt{2\phi_X^*}.
	$$
\end{theorem}
\begin{proof}
It suffices for us to prove that for any set $S$, 
\begin{eqnarray}
\frac{1}{2} \phi_2(S;{G_{X}})\leq \phi_{X}(S;G) \leq 2 \phi_2(S;G_X). \label{eqn:mix_conductance}
\end{eqnarray}

Assume for now that the inequality~(\ref{eqn:mix_conductance}) holds. By Lemma~\ref{lem:secondcheeger}, the set $T^*$ satisfies  $\phi_2(T^*;{G_{X}})\leq 2\sqrt{\psi^*}$, where $\psi^*=\min_{S\subseteq V}\phi_2(S;{G_{X}})$. Let $R$ be the set with $\phi_{X}(R;G)=\phi_X^*=\min_{S\subseteq V}\phi_{X}(S;G)$. Then by inequality~(\ref{eqn:mix_conductance}), we have 
\begin{eqnarray*}
\phi_{X}(T^*;G)&\leq& 2\phi_2(T^*;G_X)\leq 2\sqrt{\psi^*}  
\leq 2\sqrt{\phi_2(R;G_X)}\\
&\leq& 2\sqrt{2\phi_{X}(R;G)}=2\sqrt{2\phi_X^*}.
\end{eqnarray*}
This will then finish the proof. Therefore, we only need to prove the inequality~(\ref{eqn:mix_conductance}). We will make use of the Lemma 
\ref{lemma:bensor} from \cite{benson2016higher}.%

By Lemma~\ref{lemma:bensor}, we have
\begin{eqnarray*}
cut_{X} (S;G) 
&=& (1-\lambda)cut_3 (S;G) + \lambda cut_2 (S;G)\\
&=&(1-\lambda) \frac12 cut_2(S;G_T) + \lambda cut_2 (S;G),
\end{eqnarray*}

and
\begin{eqnarray*}
vol_{X}(S;G) 
&=& (1-\lambda)vol_3 (S;G)+\lambda vol_2 (S;G) \\
&=& (1-\lambda)\frac{1}{2} vol_2(S;G_T)+\lambda vol_2 (S;G).
\end{eqnarray*}

Since the adjacency matrix of $G_{X}$ is a linear combination of the adjacency matrix of $G_T$ and the adjacency matrix of $G$, i.e.,
$\matrics{W}_{X}=(1-\lambda)\matrics{W}_T+\lambda \matrics{W}$, we have  
\begin{eqnarray*}
cut_2(S;G_X)
&=&(1-\lambda)cut_2(S;G_T) +\lambda cut_2(S;G),\\
vol_2(S;G_X)
&=&(1-\lambda)vol_2(S;G_T) +\lambda vol_2(S;G).
\end{eqnarray*}

The above equations imply that for any set $S$, 
\begin{eqnarray*}
\frac12 cut_2(S;G_X)\leq cut_{X}(S;G) \leq cut_2(S;G_X),\\
\frac12 vol_2(S;G_X)\leq vol_{X}(S;G) \leq vol_2(S;G_X).
\end{eqnarray*}

The last inequality also implies that for any $S$, 
\begin{eqnarray*}
\frac12 vol_2(\bar{S};G_X)\leq vol_{X}(\bar{S};G) \leq vol_2(\bar{S};G_X).
\end{eqnarray*}

Therefore, by the definition of $\phi_{X}(S;G)$, we have 
\begin{eqnarray*}
\phi_{X}(S;G) &\leq& \frac{cut_2(S;G_X)}{\min(\frac{1}{2} vol_2(S;G_X), \frac{1}{2} vol_2(\bar{S};G_X))} \\
&=& 2 \phi_2(S;G_X),\\
\phi_{X}(S;G) &\geq& \frac{\frac{1}{2} cut_2(S;G_X)}{\min(vol_2(S;G_X), vol_2(\bar{S};G_X))} \\
&=& \frac12 \phi_2(S;G_X).
\end{eqnarray*}
This completes the proof of the inequality (\ref{eqn:mix_conductance}).
\end{proof}

\textbf{Complexity analysis.}
The computational time of MOSC-GL is dominated by the time to form $\matrics{W}_X$ and
compute the second eigenvector of $\matrics{L}_{X}$. The former requires finding all
triangles in the graph, which can be as large as $O(n^3)$ for a complete graph. While most
real networks are far from complete so the actual complexity is much lower than $O(n^3)$. For
the latter, it suffices to use power iteration to find an approximate eigenvector, with
each iteration at $\tilde{O}(g)$, where $g$ denotes the number of non-zero entries in
$\matrics{L}_{X}$.

\subsection{MOSC Based on Random Walks (MOSC-RW)}\label{MOSC-RW}
Alternatively, we can develop MOSC under the random walks framework. 
Edge/triangle conductance can be viewed as a probability corresponding to the Markov chain. For a set $S$ with edge volume at most half of the total graph edge volume, the edge conductance of $S$ is the probability that a random walk will leave $S$ conditioned upon being inside $S$, where the transition probabilities of the walk are defined by edge connections \cite{von2007tutorial}. Similarly, for a set $S$ with triangle volume at most half of the total graph triangle volume, the triangle conductance of $S$ is the probability that a random walk will leave $S$ conditioned upon being inside $S$, where the transition probabilities of the walk are defined by the triangle connections \cite{benson2015tensor}. This motives us to directly combine random walks from edge and triangle connections to perform MOSC. Therefore, we propose MOSC-RW to consider both edge and triangle structures via the respective probability transition matrix and tensor, under the random walks framework.

Specifically, starting with the third-order adjacency tensor $\tensor{T}$, we define a third-order transition tensor $\tensor{P}$ as Eq. (\ref{trans-tensor}).
Each entry of $\tensor{P}$ represents the transition probability of a random walk such that the probability of jumping to a state $\scalars{j}$ depends on the last two states $\scalars{i}$ and $\scalars{k}$ \cite{xu2016representing}. In the case $\sum_{m=1}^n \tensor{T}(i, m, k) = $ 0, 
we set $\tensor{P} (i, j, k)$ with $0$.

Let $\matrics{T}_k\in\mathbb{ R}^{n\times n}$ denote the $k$th $n\times n$ block of $\tensor{P}$, i.e., %
\begin{equation}
\matrics{T}_k=\tensor{P}(:, :, k).
\end{equation}
Next, we average $\{\matrics{T}_k, k=1, ..., n\}$ to reduce $\tensor{P}$ to a similarity matrix $\matrics{A}$:
\begin{equation}\label{MOSCave}
\matrics{A}  = \frac{1}{n}\sum_{k = 1} ^{n}\matrics{T}_k.
\end{equation}
Now recall that $\matrics{P}=\matrics{D}^{-1}\matrics{W}$ denotes the probability transition matrix of random walks on the input graph. We construct a mixed-order similarity
matrix $\matrics{H}$ by a weighted sum of 
$\matrics{A}$ and $\matrics{P}$ via a mixing parameter $\lambda \in [0, 1]$ as:
\begin{equation}\label{MOSCeqn}
\matrics{H} = (1 - \lambda) \matrics{A} + \lambda \matrics{P}.
\end{equation}
Thus, we obtain the MOSC-RW algorithm with standard SC steps on $\matrics{H}$, as summarised in Algorithm \ref{alg:MOSC-RW}. 

When $\lambda=1$, MOSC-RW is equivalent to conventional SC by Shi and Meila \cite{shi2000normalized} and considers only second-order structures.
MOSC-RW with $\lambda=0$ considers only third-order structures, which is a simplified (unweighted) version of tensor SC (TSC) by Benson \etal \cite{benson2015tensor}, so we name it as simplified TSC (STSC). In the intermediate case, $\lambda$ controls the trade-off. 

\textbf{Interpretation.} Now we interpret the model (Eq. (\ref{MOSCeqn})) as a mixed-order random walk process. At every step, the random walker chooses either a first-order (with probability $\lambda$) or a second-order (with probability ($1- \lambda$)) random walk. For the first-order random walk, the walker jumps from the current node $i$ to a neighbour $j$ with probability $\matrics{P}(i,j)=\frac{1}{\matrics{D}(i,i)}$. For the second-order random walk in $\matrics{A}$, $\matrics{A}(i,j)$ is the probability of the following random process: supposing the walker is at vertex $i$, it first samples a vertex $k$ with probability $\frac{1}{n}$, then in the case that some neighbour $k$ of $i$ is sampled and $i,j,k$ forms a triangle, the walker jumps from $i$ to $j$ with probability $1/\matrics{W}_T(i,k)$, where $\matrics{W}_T(i,k)$ is the number of triangles containing both $i$ and $k$.

\begin{algorithm}[t]
	\caption{MOSC-RW}
	\begin{algorithmic} [1]\label{alg:MOSC-RW}
		\REQUIRE {$\graset{G}$ = ($\graset{V}$, $\graset{E}$)}, a mixing parameter $\lambda$
		\ENSURE Two node sets $\{S, \bar{S}\}$
		\STATE Construct the adjacency matrix $\matrics{W}\in 
		\mathbb{ R}^{n\times n}$.
		\STATE Construct the adjacency tensor $\tensor{T} \in 
		\mathbb{ R}^{n\times n\times n}$. 		 
		\FOR{$1 \leqslant i,j,k\leqslant \emph{n}$} 
		\IF{$\sum_{m=1} ^n\tensor{T}(i, m, k) \neq 0$} 
		    \STATE	$\tensor{P} (i, j, k) =$ $\tensor{T}(i, j, k)$ $/\sum_{m=1}^n$$\tensor{T}(i, m, k)$.
		\ELSE 
		    \STATE $\tensor{P} (i, j, k) = $ 0.
		\ENDIF
		 
		\ENDFOR
		\STATE $\matrics{T}_k$ $\gets$ $\tensor{P}(:, :, k)$ for $k=1,\cdots,n$.
		\STATE Compute the reduced similarity matrix $\matrics{A}$ via (\ref{MOSCave}). %
		\STATE Let $\matrics{D}$ be diagonal with $\matrics{D}_{ii}$ = $\sum_i^n \matrics{W}(i,j)$. 
		\STATE $\matrics{P}$ = $\matrics{D}$$^{-1}$$\matrics{W}$.
		\STATE $\matrics{H}$ =  $(1 - \lambda)\matrics{A} + \lambda \matrics{P}$.
		\STATE Compute the second largest eigenvector $\vectors{v}$ of $\matrics{H}$.
		\STATE $\vectors{v}  \gets$ Sorting entries of $\vectors{v}$.
		\STATE  \{$\grasubset{S}$, $\bar{\grasubset{S}}$\} $\gets$ Sweep cut on $\vectors{v}$ w.r.t. some cut criteria. 
	\end{algorithmic}
\end{algorithm}

\textbf{Complexity analysis.} 
The running time of MOSC-RW is again dominated by the time of finding all the triangles and the approximate eigenvector, and thus asymptotically the same as the running time of MOSC-GL. However, since MOSC-RW involves tensor construction, normalisation and averaging, it is more complex than MOSC-GL in implementation.  %

\subsection{Multiple Clusters and Higher-order Cheeger Inequalities of MOSC}
To cluster a network into $k > 2$ clusters based on 
mixed-order structures, MOSC-GL and MOSC-RW follow the conventional SC \cite{von2007tutorial}. Specifically, MOSC-GL treats the first $k$ row-normalised  eigenvectors of $\matrics{L}_X$ 
as the embedding of nodes that can be clustered by $k$-means. Similarly, MOSC-RW uses the first $k$ eigenvectors of $\matrics{H}$ as the node embedding to perform $k$-means.

Regarding performance guarantee, following \cite{benson2016higher} and \cite{lee2014multiway},
MOSC-GL and MOSC-RW do not have performance guarantee with respect to higher-order
Cheeger inequalities. However, by replacing $k$-means with a different clustering algorithm, MOSC-GL can derive a theoretical performance guarantee~\cite{lee2014multiway}.

\subsection{Automatic Determination of $\lambda$}\label{autoLambda}
The mixing parameter $\lambda$ is the only hyperparameter in MOSC. 
To improve the usability, we design schemes to automatically determine its optimal value
$\lambda^*$ from a set $\Lambda$ based on the quality of output
clusters~\cite{leskovec2010empirical,chakraborty2017metrics,yang2015defining}. For
bi-partitioning networks, the cut criterion 
 used to obtain output clusters can help to determine the best $\lambda^*$ from $\Lambda$. For multiple partitioning networks, 
we can use the sum of triangle densities of the individual cluster to determine the best $\lambda^*$ from $\Lambda$. 

Specifically, for each $\lambda'\in \Lambda$, let \{$\grasubset{S}_{\lambda'}$,~$\overline{\grasubset{S}_{\lambda'}}$\} denote the MOSC bi-partitioning clusters obtained with $\lambda=\lambda'$. 
For a specific minimisation or maximisation cut criterion $\tau$ (e.g., edge conductance $\phi_2$), we choose  $\lambda$ to be  the one that optimises $\tau$, i.e.,  
\begin{equation}
\lambda^* =  \arg \ \min _{\lambda' \in \Lambda} \tau(\grasubset{S}_{\lambda'}) \:\: \mathrm{or} \:\: \lambda^* =  \arg \ \max _{\lambda' \in \Lambda} \tau(\grasubset{S}_{\lambda'}),
\end{equation}
respectively. 

For the case of multiple partitions, we propose a triangle-density-based scheme to determine $\lambda$ as follows:
\begin{equation}
\lambda^* = \arg \ \max _{\lambda' \in \Lambda} \ \sum_{c = 1}^k \frac{\sum_{\gravex{v_i,v_j,v_k  \in \grasubset{S}_c(\lambda')}} \tensor{T}(i, j, k)}{6|\grasubset{S}_c(\lambda')|},
\end{equation}
where $\grasubset{S}_c(\lambda')$ denotes the $c$-th cluster resulted from $\lambda'$, 
and the
factor 1/6 is used to avoid repeated count of triangles in an undirected graph.

\subsection{Structure-Aware Error Metrics}\label{secENode} 
If we have ground-truth clusters available, we can use them to measure  performance of clustering algorithms. Existing works commonly use mis-clustered nodes \cite{ghoshdastidar2017uniform} or related metrics (e.,g. NMI)~\cite{ana2003robust}. We denote the ground-truth partition of $\graset{G}$ with $k$ clusters as  $\mathbb{S}^* = \{S_1^*, S_2^*,\dots, S_k^*\}$ and a candidate partition to be evaluated as $\mathbb{S} = \{S_1, S_2,\dots, S_k\}$. %
The mis-clustered node metric is defined as 
\begin{eqnarray}
\epsilon_{N}(\mathbb{S}^*, \mathbb{S}) = 
\underset{\sigma}{\min} \sum_{c=1}^k |  \grasubset{S}_c^* \oplus \grasubset{S}_{\sigma(c)}|, \label{eqn:eps_N}
\end{eqnarray}
which measures the difference between two partitions $\mathbb{S}^*$ and $\mathbb{S}$, where $\sigma$ indicates all possible permutations of \{1, 2, $\dots$, $k$\} and $\oplus$ denotes the symmetric difference between the two corresponding sets. A smaller $\epsilon_{N}$ indicates a more accurate partition.

A limitation of the above metric is that it fails to truly reflect the errors made in preserving structures such as edges  or triangles. Our  studies show that mis-clustered nodes do not have a \textit{monotonic} relationship with mis-clustered edges or triangles. That is, a smaller number of mis-clusterd nodes does not imply smaller number of mis-clustered edges or triangles, and vice versa. This motivates us to propose two new error metrics $\epsilon_{E}$ and $\epsilon_{T}$ that measure the \textit{mis-clustered edges} and \textit{triangles}, respectively. These new metrics can provide more insights in the preservation of edges and triangles. 

Specifically, we define $\epsilon_{E}$ as %
\begin{equation}
\epsilon_{E}(\mathbb{S}^*,\mathbb{S})  = \sum_{c=1}^k E_N  (\grasubset{S}^*_c) - 
\underset{\sigma}\max \sum_{c=1}^k  E_N  (\grasubset{S}^*_c \cap \grasubset{S}_{\sigma(c)})
\label{ES},
\end{equation}
where $E_N(S)$ is the number of edges in $S$. We can define $\epsilon_{T}$ similarly by replacing $E_N(S)$ in Eq. (\ref{ES}) with $T_N(S)$, where $T_N(S)$ is the number of triangles in $S$.

\subsection{Sweep Cut Error Lower Bound}\label{secOcut} Supposing we have the ground-truth partition $\mathbb{S}^*$ and the \textit{sorted} eigenvector $\vectors{v}$, we can examine the best possible partition, named as the \textit{optimal cut} (Ocut). This Ocut can be obtained by performing sweep cut using an error metric ($\epsilon_N$, $\epsilon_E$ or $\epsilon_T$) computed from  partitions as the cut criteria, e.g., for $\epsilon_N$, its Ocut is 
\begin{eqnarray}
Ocut_N(\mathbb{S}^*)= \arg\min_{\grasubset{T}_u: 1\leq \scalars{u} \leq \scalars{n}-1} \ \epsilon_N (\mathbb{S}^*, \{ \grasubset{T}_u, \overline{\grasubset{T}_u} \}),
\label{Ocut}
\end{eqnarray}
where $T_u$ is the first $u$ nodes from the sorted $\vectors{v}$.
Replacing $\epsilon_N$ with $\epsilon_E$ or $\epsilon_T$ gives $Ocut_E(\mathbb{S}^*)$ or $Ocut_T(\mathbb{S}^*)$. It is important to note that Ocut does not depend on any cut criterion.

Ocut can help us understand the potential and model quality of a particular SC algorithm. The computed Ocut can serve as its lower bound, indicating how far a partition result (without looking at the ground-truth) is from the best possible solution (looking at the ground-truth given the fixed sorted node order), or equivalently, how much potential improvement is possible. %

%% file: Experiments.tex
\section{Experiments} 
This section aims to evaluate MOSC against existing SC methods. In addition, we will examine the effect of cut criteria and hyperparameter $\lambda$, and gain insights from the newly designed error metrics and optimal cut. 

\subsection{Experimental Settings}

\begin{table*}[!t]
	\caption {Statistics of the 2,005 networks. The number in parentheses is the median for each range.}\label{statistics}
	\centering
	\begin{tabular}{lcc|cccccc}
		\toprule 
		\textbf{Network} & \textbf{$|\graset{V}|$} & \textbf{$|\graset{E}|$}&
		  {\textbf{Size}} &{\textbf{Triangle density}}&{\textbf{\#Interaction edges}}&{\textbf{\#Clusters/network}}&{\textbf{\#Network(s)}} \\ 
		\midrule
		DBLP & 317K&1.05M& 14$\sim$303 (22)& 7.4$\sim$167.9 (15.4)&1$\sim$278 (15)&2&500 \\
		YouTube & 1.13M&2.99M& 6$\sim$389 (91)& 1$\sim$22.9 (3.73)&1 $\sim$1054 (89)&2&500\\
		Orkut& 3.07M&117M& 88$\sim$379 (206) & 213.7$\sim$1526 (452.6) &37$\sim$10470 (2411) &2&500\\
		LJ & 4.00M&34.7M& 33$\sim$193 (98) & 116.3$\sim$2968 (422.4)&1$\sim$9179 (1489)&2 &500\\
		\hline
		Zachary & 34&78&34 & 1.32& 11&2&1 \\
		Dolphin & 62&159& 62&1.53& 6 &2&1\\
		Polbooks & 105&441&105 & 5.33& 70 &3&1\\
		Football & 115&613& 115& 7.04& 219 &12&1\\
		PBlogs & 1490&16716& 1490& 67.8& 1576&2 &1\\
		\bottomrule
	\end{tabular}
\end{table*}

\textbf{Datasets.} 
The experiments were conducted on two popular groups of networks
with very different triangle densities: 
1) five full real-world networks: Zachary's karate club
 (Zachary)~\cite{zachary1977information}, Dolphin social
network (Dolphin)~\cite{lusseau2003emergent}, American college football (Football)~\cite{newman2006modularity}, U.S. politics books (Polbooks)~\cite{newman2006modularity} and Political blogs (PBlogs)~\cite{adamic2005political};
2) four complex real-world networks: DBLP, YouTube, Orkut, and LiveJournal (LJ) from
the Stanford Network Analysis Platform
(SNAP) \cite{yang2015defining}.\footnote{\url{https://snap.stanford.edu/data/index.html}}
 All networks have ground-truth communities available.
For the four 
SNAP networks, we extract paired communities to focus on bi-partitioning problems with the following procedures: 
\begin{enumerate}
	\item For each network, we select communities with the top 500 highest triangle densities, among those communities having no more than 200 nodes (for DBLP, YouTube, and Orkut) or 100 nodes (for LJ) because it has high density;
	\item For every community in the top list, we choose another community having the most connections with it, among all the other communities in the respective network (without limiting the community node size). These two communities form a bi-partitioning network.
\end{enumerate}
In this way, we extracted 2,000 networks from SNAP. The statistics of networks are summarised in Table \ref{statistics}. 

\textbf{Compared algorithms.} 
We evaluate MOSC-GL and MOSC-RW against the following six state-of-the-art methods, including both edge-based SC and triangle-based SC, and both 
global and local methods.
\begin{enumerate}
\item{SC-Shi \cite{shi2000normalized}:} 
Shi and Malik developed a method aiming to minimise Ncut$_2$ criterion via a generalised eigenvalue problem of~Eq.~(\ref{trans_matrix}). 

\item{SC-Ng \cite{ng2002spectral}:} Ng \etal  designed a method built upon~\cite{shi2000normalized}. Instead of using one dominant eigenvector,
it used the first $k$ eigenvectors of $\matrics{L}$ for performing $k$ partitions and then an additional row normalisation step before $k$-means.

\item{Tensor Spectral Clustering (TSC) \cite{benson2015tensor}}:  TSC
is a higher-order spectral clustering method developed by Benson \etal They constructed a transition tensor $\tensor{P}$ as in Eq. (\ref{trans-tensor}) and used an expensive multilinear PageRank algorithm~\cite{gleich2015multilinear} to produce a vector as the weight for reducing the tensor to a matrix via weighted average, followed by conventional SC. 

\item{Higher-order SVD (HOSVD) \cite{ghoshdastidar2014consistency}}: To address the hypergraph clustering problem, 
this method used an adjacency tensor $\tensor{T}$ to encode hyperedge, which is equivalent to the adjacency tensor definition in Eq. (\ref{adj-tensor}). $\tensor{T}$ is then reduced to a matrix via computing a modelwise covariance matrix, followed by conventional SC.

\item{Motif-based SC (MSC)  \cite{benson2016higher} / Tensor Trace Maximisation (TTM) \cite{ghoshdastidar2017uniform}}: 
MSC is a general higher-order spectral clustering method via re-weighting edges according to the number of motifs containing corresponding edges, followed by conventional SC. TTM is independently proposed but equivalent to MSC, which we have verified both analytically and experimentally.

\item{HOSPLOC}\cite{zhou2017local}: 
This is a higher-order local clustering method aiming for more efficient processing while taking higher-order network structures into account. 
\end{enumerate}

\begin{figure*} [t]
	\centering
	\subfigure[$\epsilon_N$ on YouTube]{\includegraphics[width=58mm]{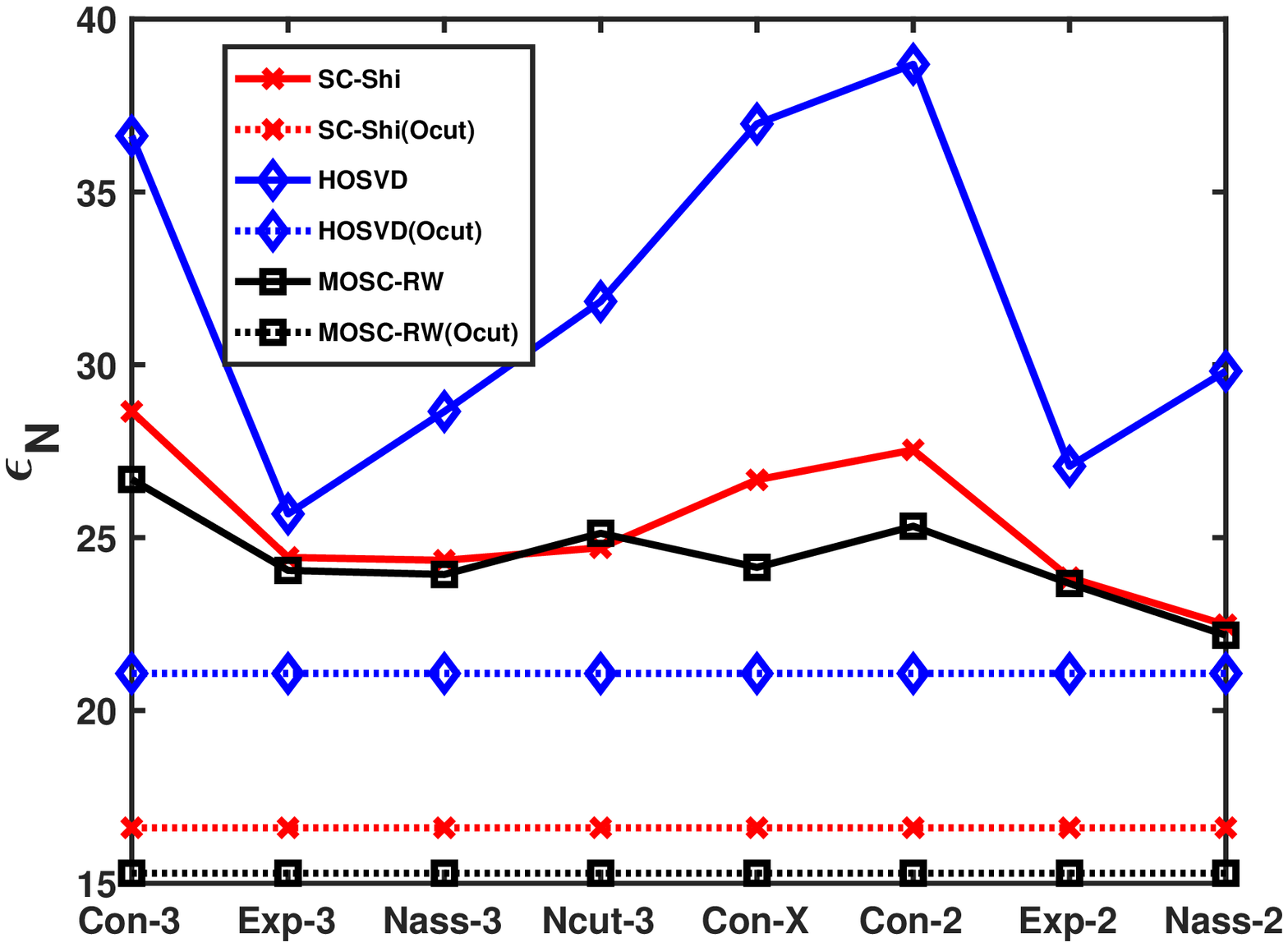}\label{cutfig:subfig1}}
	\subfigure[$\epsilon_T$ on YouTube]{\includegraphics[width=58mm]{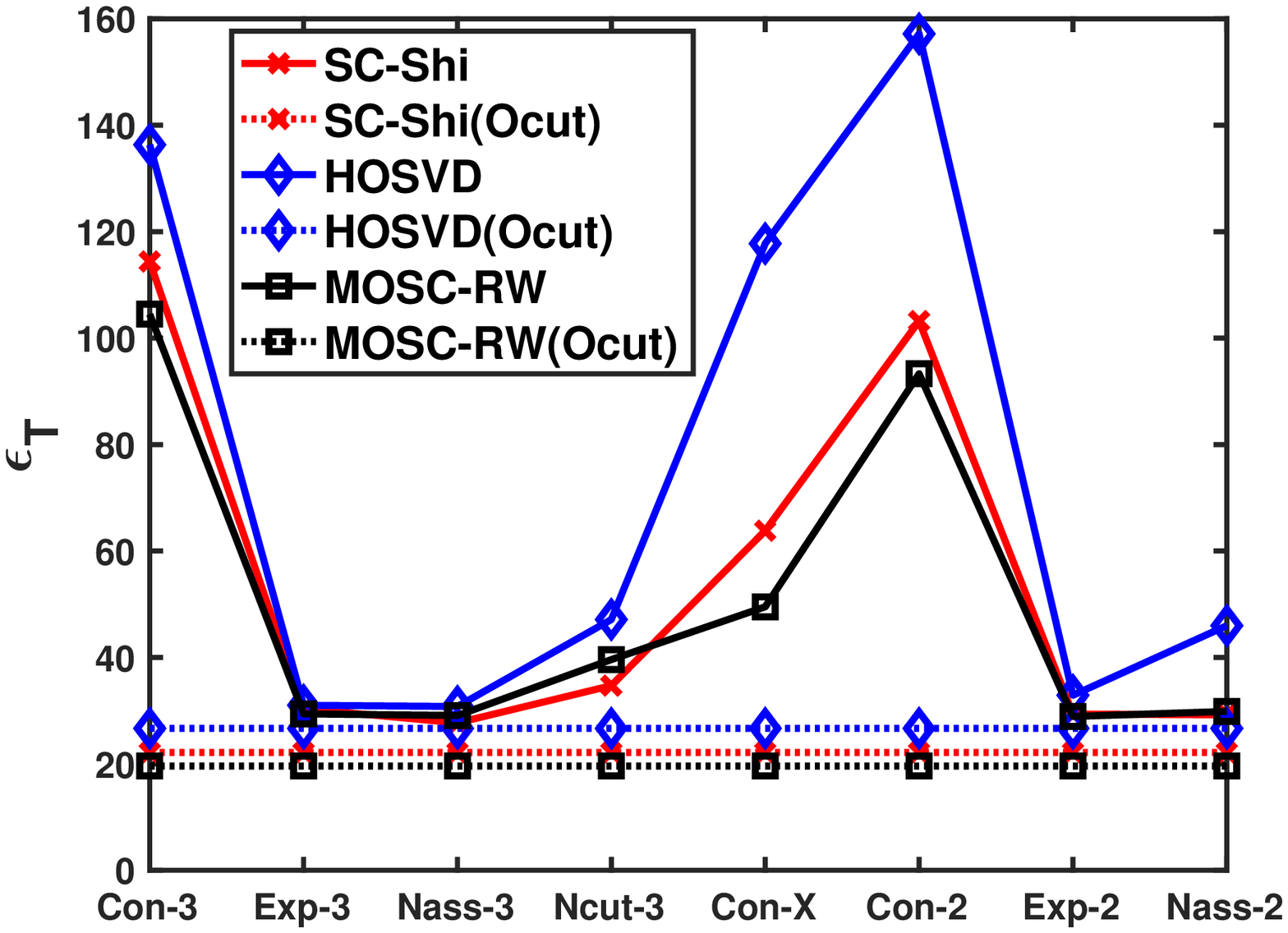}\label{cutfig:subfig3}}
	\subfigure[Non-monotonic relationship on YouTube]{\includegraphics[width=52mm]{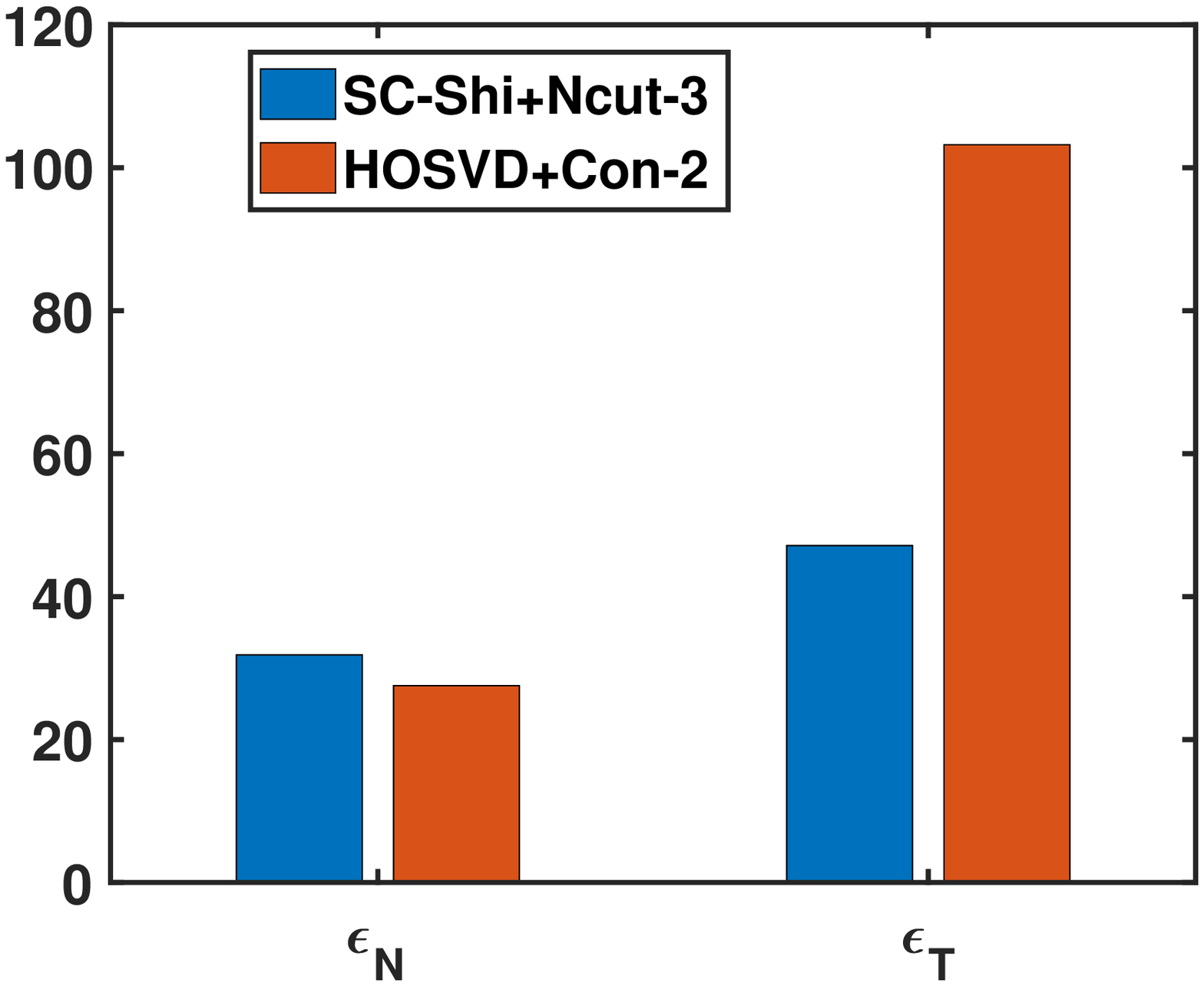}\label{cutfig:subfig4}}
	\caption{Clustering performance of SC-Shi, HOSVD, and MOSC-RW($\lambda=0.5$) in terms of mis-clustered nodes ($\epsilon_{N}$) and triangles ($\epsilon_{T}$) for eight cut criteria: Conduct$_3$(Con-3), Exp$_3$, Nassoc$_3$(Nass-3), Ncut$_3$, Conduct-X(Con-X),
	Conduct$_2$(Con-2), Exp$_2$, Nassoc$_2$(Nass-2). Their lower bounds are provided by the optimal cut (the horizontal dotted lines). We can observe the best criterion (Nass$_3$) for $\epsilon_T$ for all algorithms, non-monotonic relationship between HOSVD+Ncut$_3$ vs. SC-Shi+Con$_2$ on $\epsilon_{N}$ and $\epsilon_{T}$ as shown in \ref{multi_cut}\subref{cutfig:subfig4}, and the superiority of Ocut for MOSC-RW.}\label{multi_cut} 
\end{figure*}

We study three versions for each MOSC:
\begin{enumerate}
\item{MOSC ($\lambda=0.5$): MOSC with a fixed (recommended) $\lambda$ value of $0.5$;}
\item{MOSC (Auto-$\lambda$): MOSC with automatically determined $\lambda$;}
\item{MOSC (O-$\lambda$): This version of MOSC  varies $\lambda$ (with step 0.1) to find the optimal $\lambda$ (O-$\lambda$) that gives the smallest clustering error (using the ground truth). It provides a reference of the best possible results, showing the potential of the MOSC model.}
\end{enumerate}
Additionally, for MOSC-RW, we study simplified TSC (STSC) when $\lambda=0$.

\textbf{Evaluation metrics.} We use the proposed structure-aware metrics, mis-clustered edges ($\epsilon_{E}$) and triangles
($\epsilon_{T}$). We also use two popular metrics, mis-clustered nodes (Eq. (\ref{eqn:eps_N}))
and normalised mutual information (NMI) \cite{ana2003robust,chang2017community}. For the SNAP networks, we show the average results of the 500 bi-partitioning networks.

To define NMI,
we need the Shannon entropy for $\mathbb{S}$ that can be defined as $H(\mathbb{S}) = -\sum^k_{c=1}(n_{S_c}/n)\log(n_{S_c}/n)$, where $n_{S_c}$ is the number of vertices in community $S_c$. The mutual information between $\mathbb{S}$ and $\mathbb{S}^*$ can be expressed as
\begin{equation}
I(\mathbb{S},\mathbb{S}^*) = \sum_{c = 1}^k\sum_{d=1}^k \frac{n_{S_cS^*_d}}{n}\log(\frac{n_{S_cS^*_d}/n}{(n_{S_c}/n)
\times(n_{S^*_d}/n)}),
\end{equation}
where $n_{S_cS^*_d}$ is the number of vertices shared by communities 
$S_c$ and $S_d^*$. The NMI between two partitions 
$\mathbb{S}$ and $\mathbb{S}^*$ is defined as 
\begin{equation}
\text{NMI}(\mathbb{S},\mathbb{S}^*) = \frac{2I(\mathbb{S},\mathbb{S}^*)}{H(\mathbb{S})+H(\mathbb{S}^*)}.
\end{equation}
If $\mathbb{S}$ and $\mathbb{S}^*$ are identical, $\text{NMI}(\mathbb{S},\mathbb{S}^*)$ = 1. If $\mathbb{S}$ and $\mathbb{S}^*$ are independent, $\text{NMI}(\mathbb{S},\mathbb{S}^*)$ = 0.

\textbf{Reproducibility.} We implemented compared algorithms using Matlab code released by the authors of  MSC,\footnote{\url{https://github.com/arbenson/higher-order-organization-matlab}} HOSVD,\footnote{\url{http://sml.csa.iisc.ernet.in/SML/code/Feb16TensorTraceMax.zip}} HOSPLOC,\footnote{\url{http://www.public.asu.edu/~dzhou23/Code/HOSPLOC.zip}} and TSC via multilinear PageRank.\footnote{\url{https://github.com/dgleich/mlpagerank}} We followed guidance from the original papers to set their hyperparameters. All experiments
were performed on a Linux machine with one 2.4GHz Intel Core and 16G memory.
 We have released the Matlab code for MOSC.\footnote{\url{https://bitbucket.org/Yan_Sheffield/mosc/}} %

\begin{figure} [!t]
\centering
\subfigure[NMI on DBLP]{\includegraphics[width=28mm]{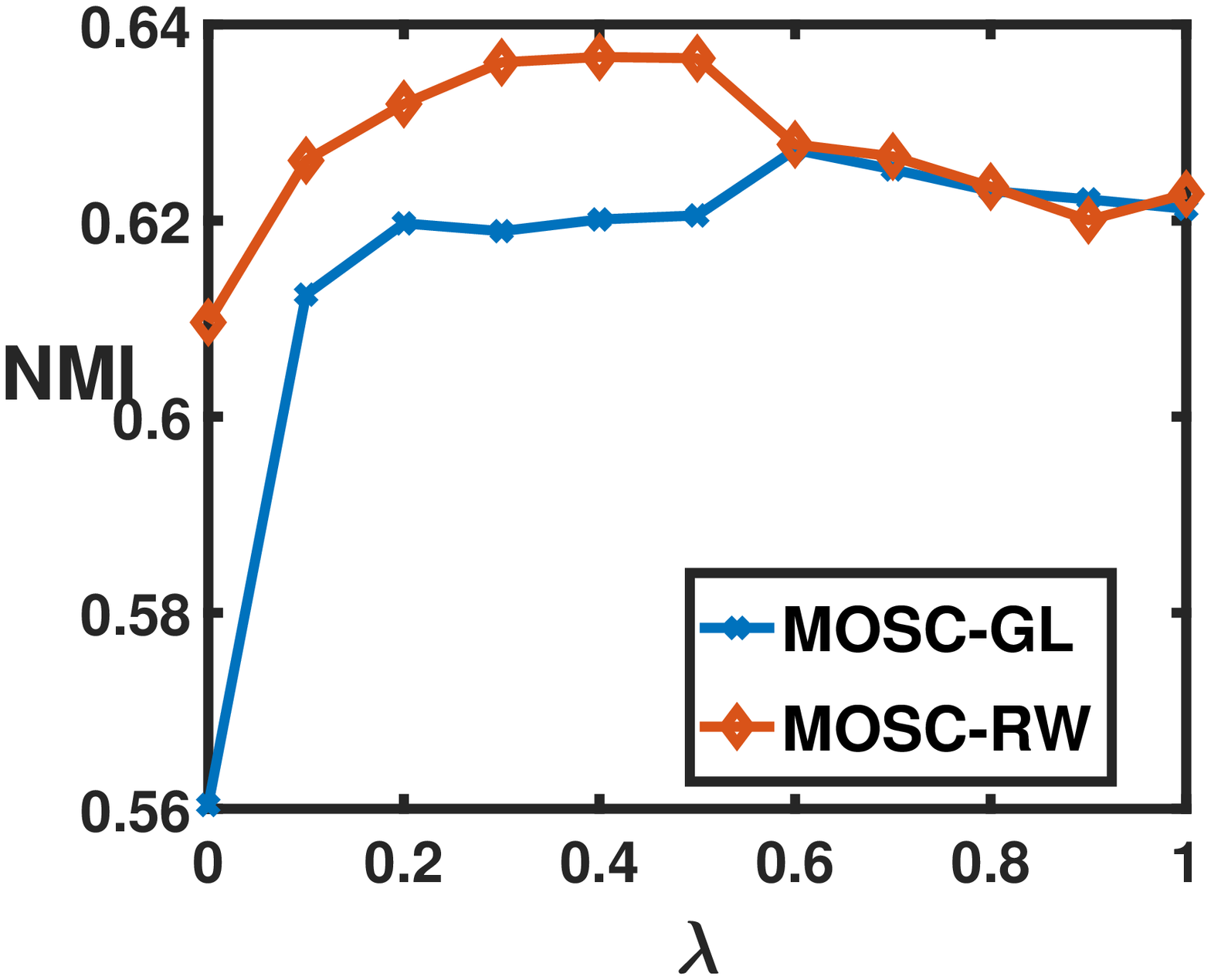}\label{senfig:subfig2}}
\subfigure[$\epsilon_N$ on DBLP]{\includegraphics[width=28mm]{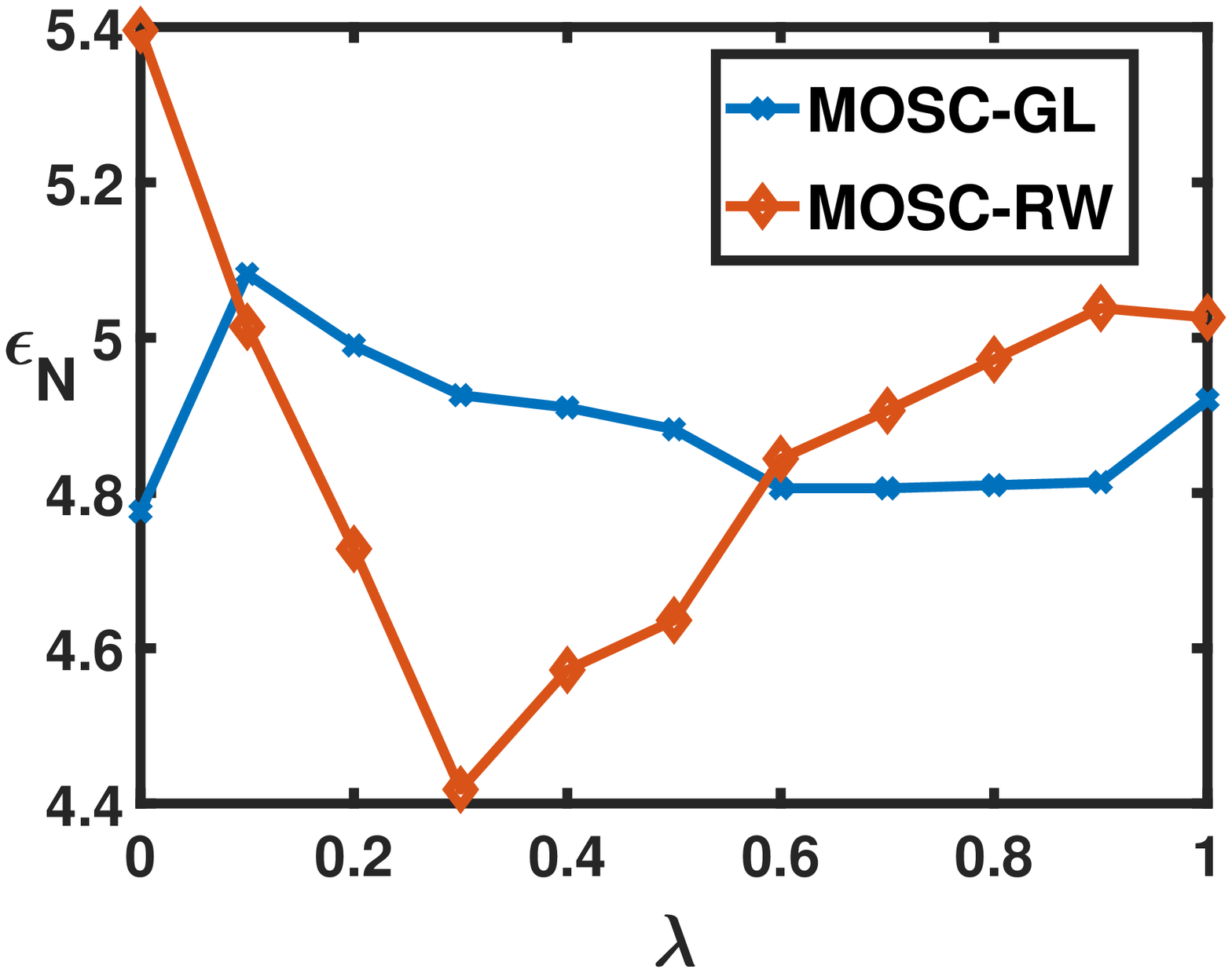}\label{senfig:subfig1}}
\subfigure[$\epsilon_T$ on DBLP]{\includegraphics[width=28mm]{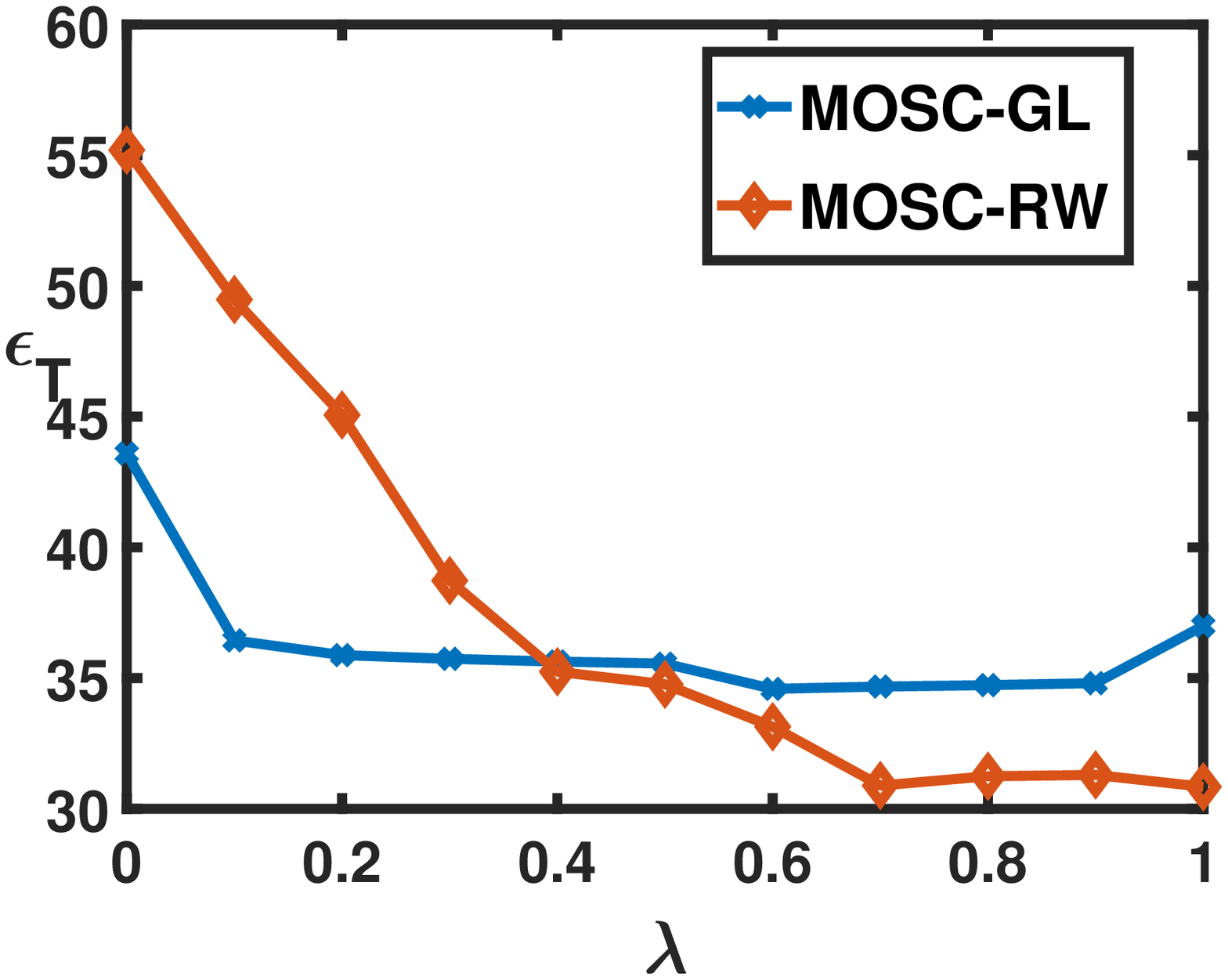}\label{senfig:subfig3}}
\subfigure[NMI on Orkut]{\includegraphics[width=28mm]{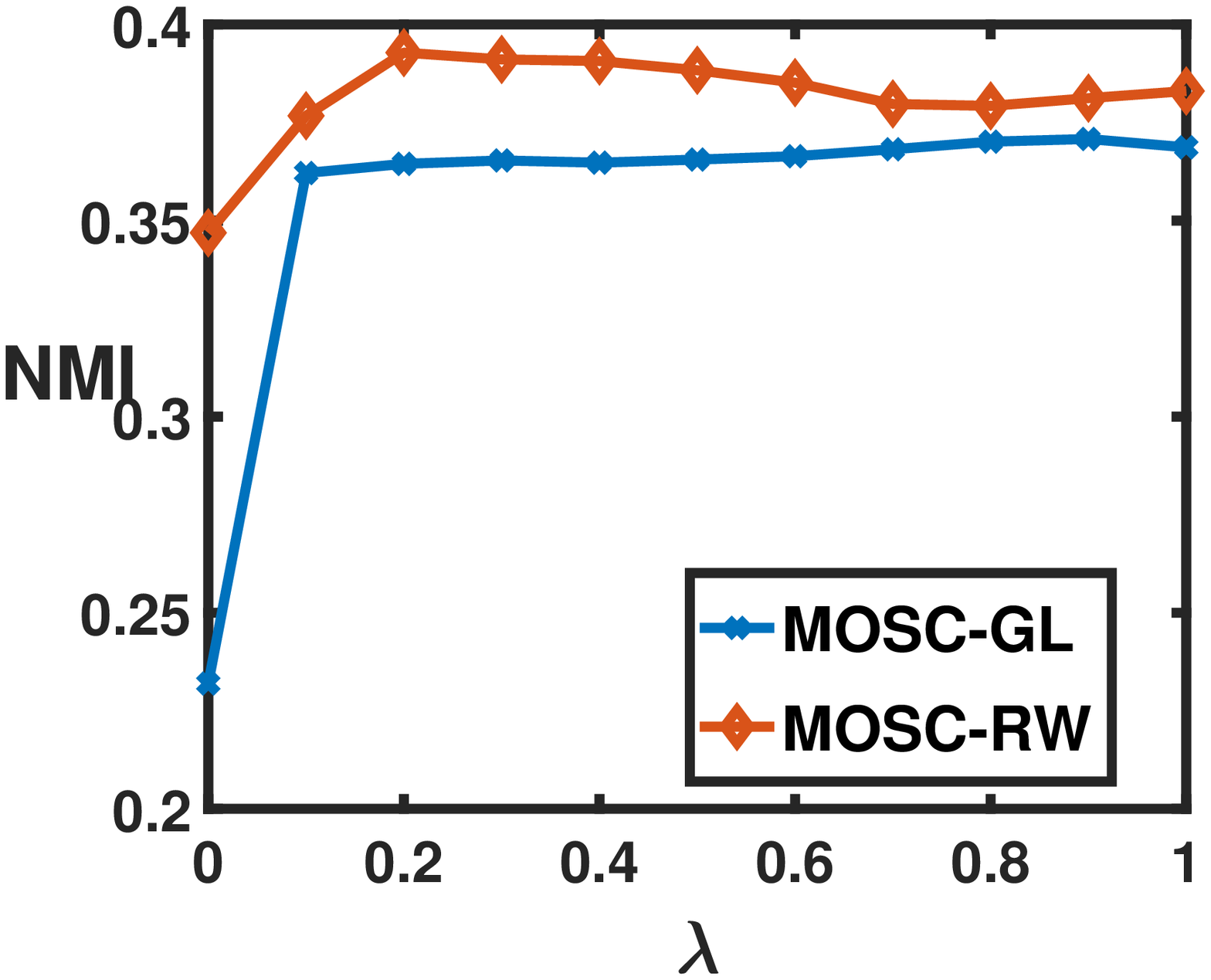}\label{senfig:subfig5}}
\subfigure[$\epsilon_N$ on Orkut]{\includegraphics[width=28mm]{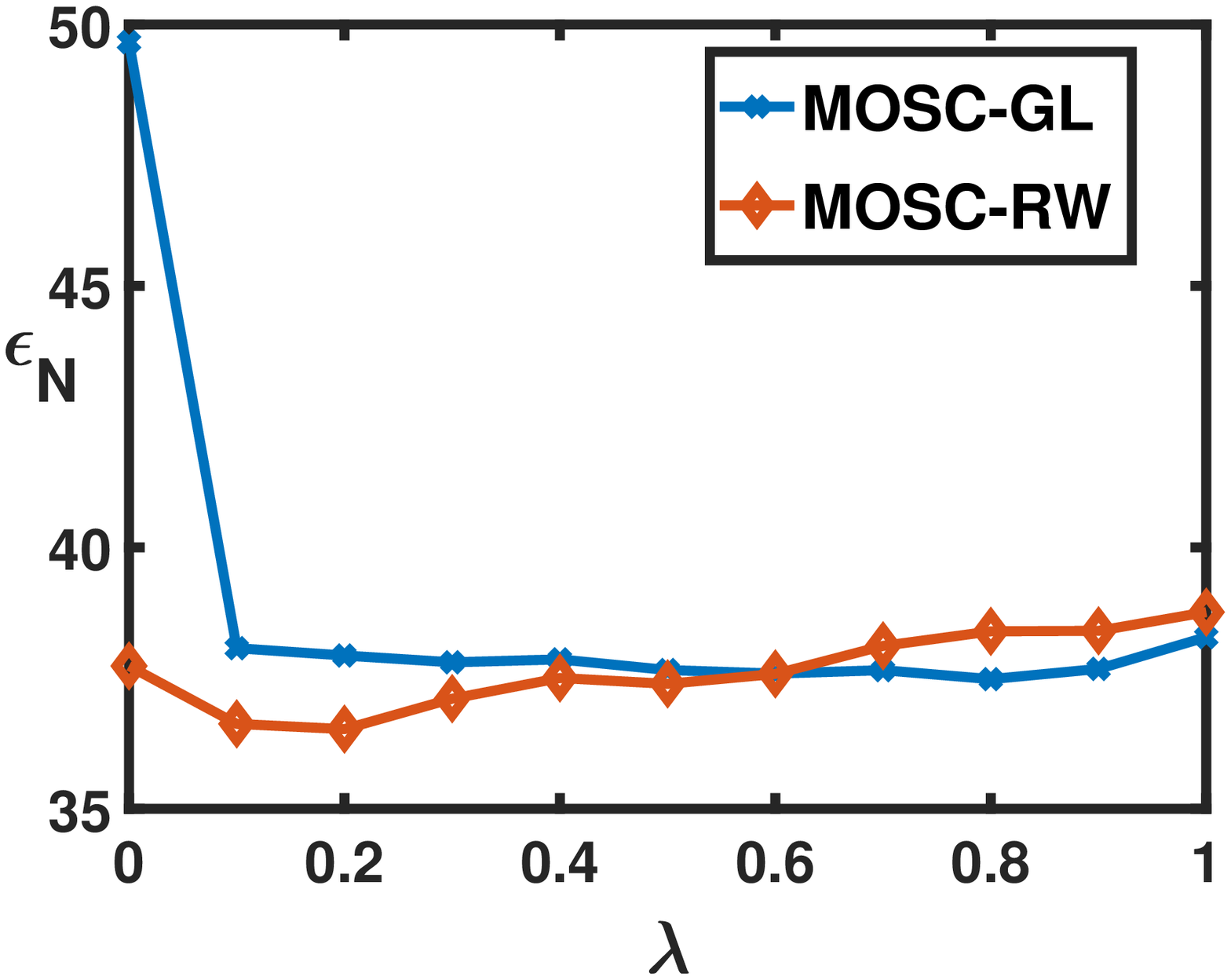}\label{senfig:subfig4}}
\subfigure[$\epsilon_T$ on Orkut]{\includegraphics[width=28mm]{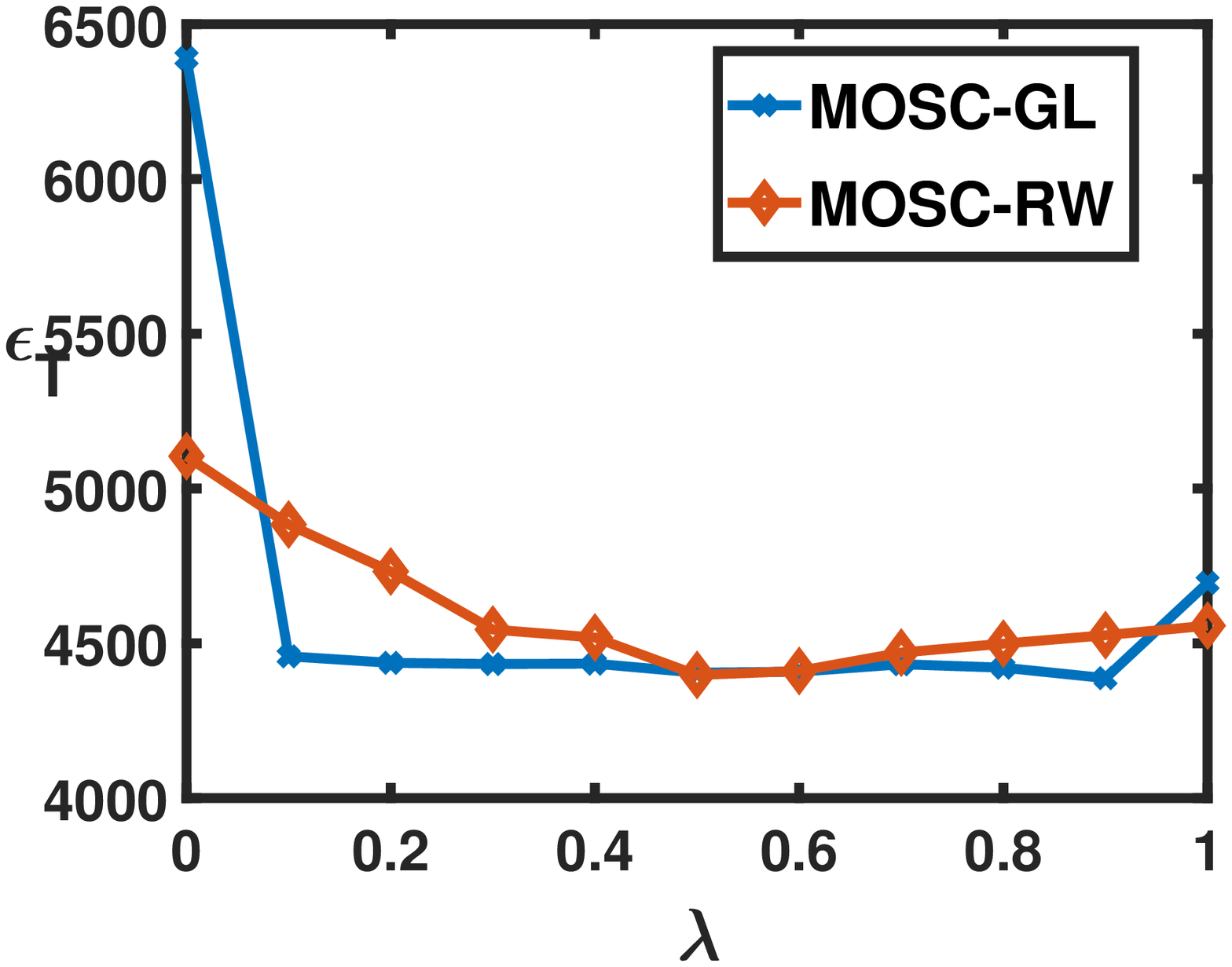}\label{senfig:subfig6}}
\subfigure[NMI on Football]{\includegraphics[width=28mm]{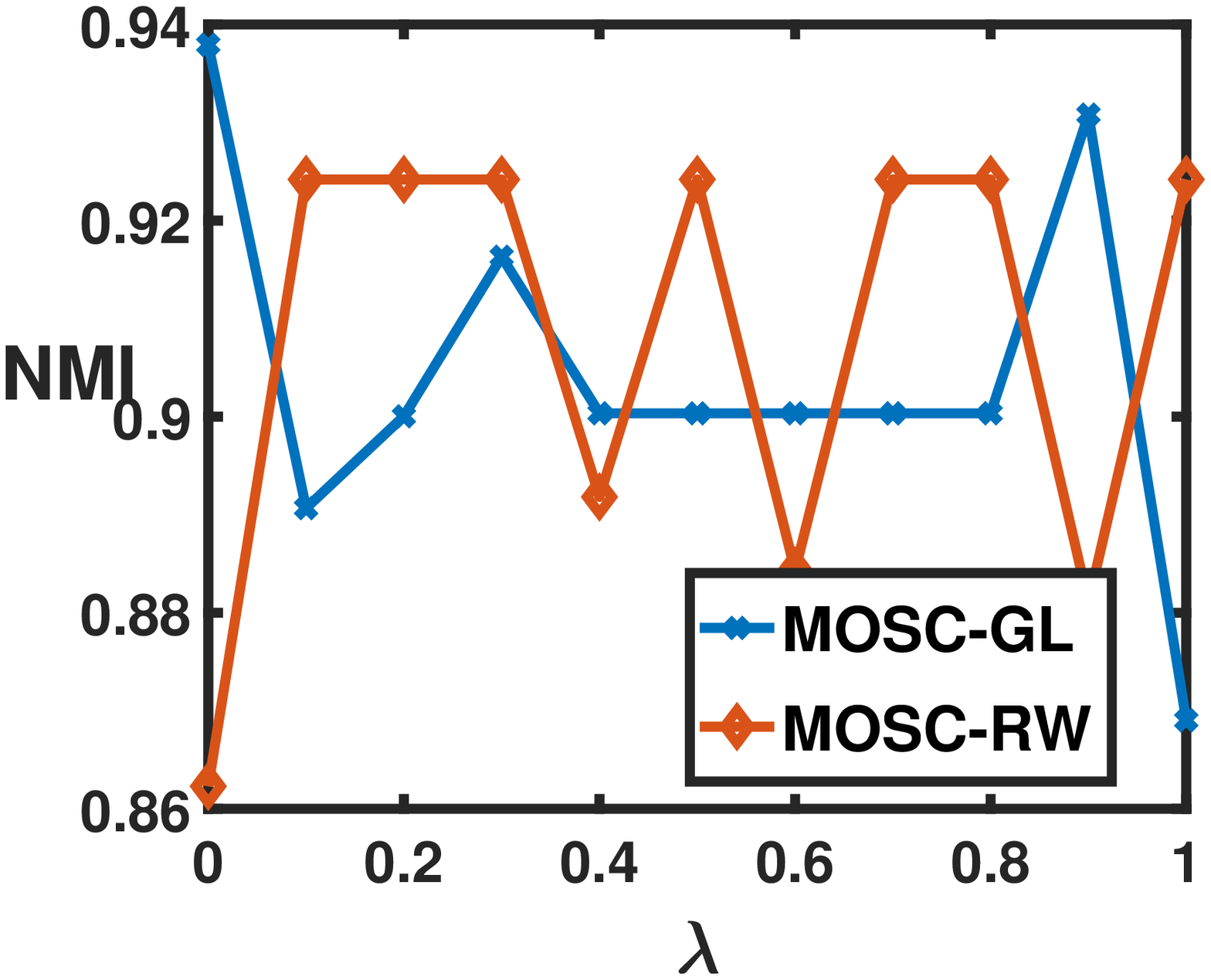}\label{senfig:subfig8}}
\subfigure[$\epsilon_N$ on Football]{\includegraphics[width=28mm]{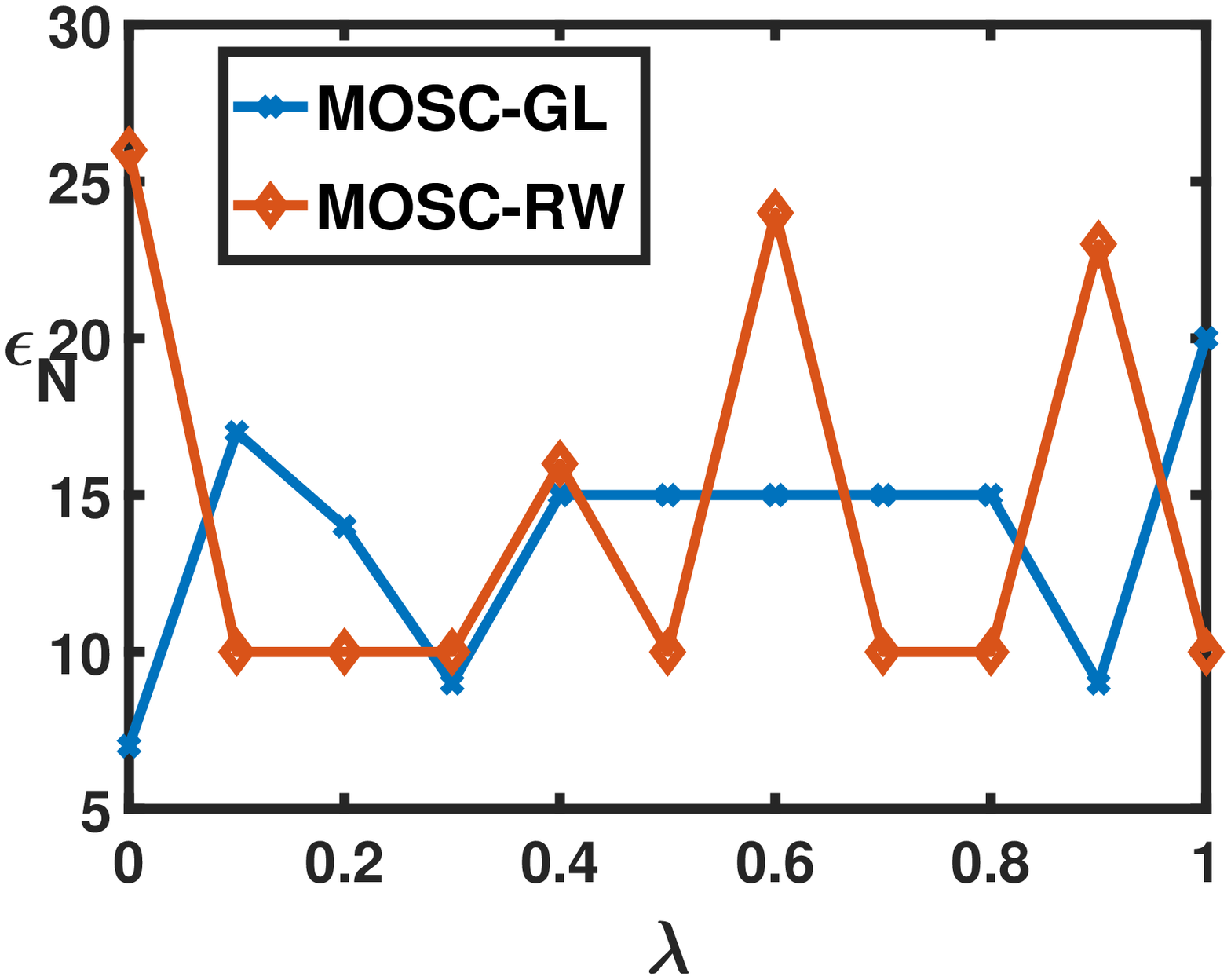}\label{senfig:subfig7}}
\subfigure[$\epsilon_T$ on Football]{\includegraphics[width=28mm]{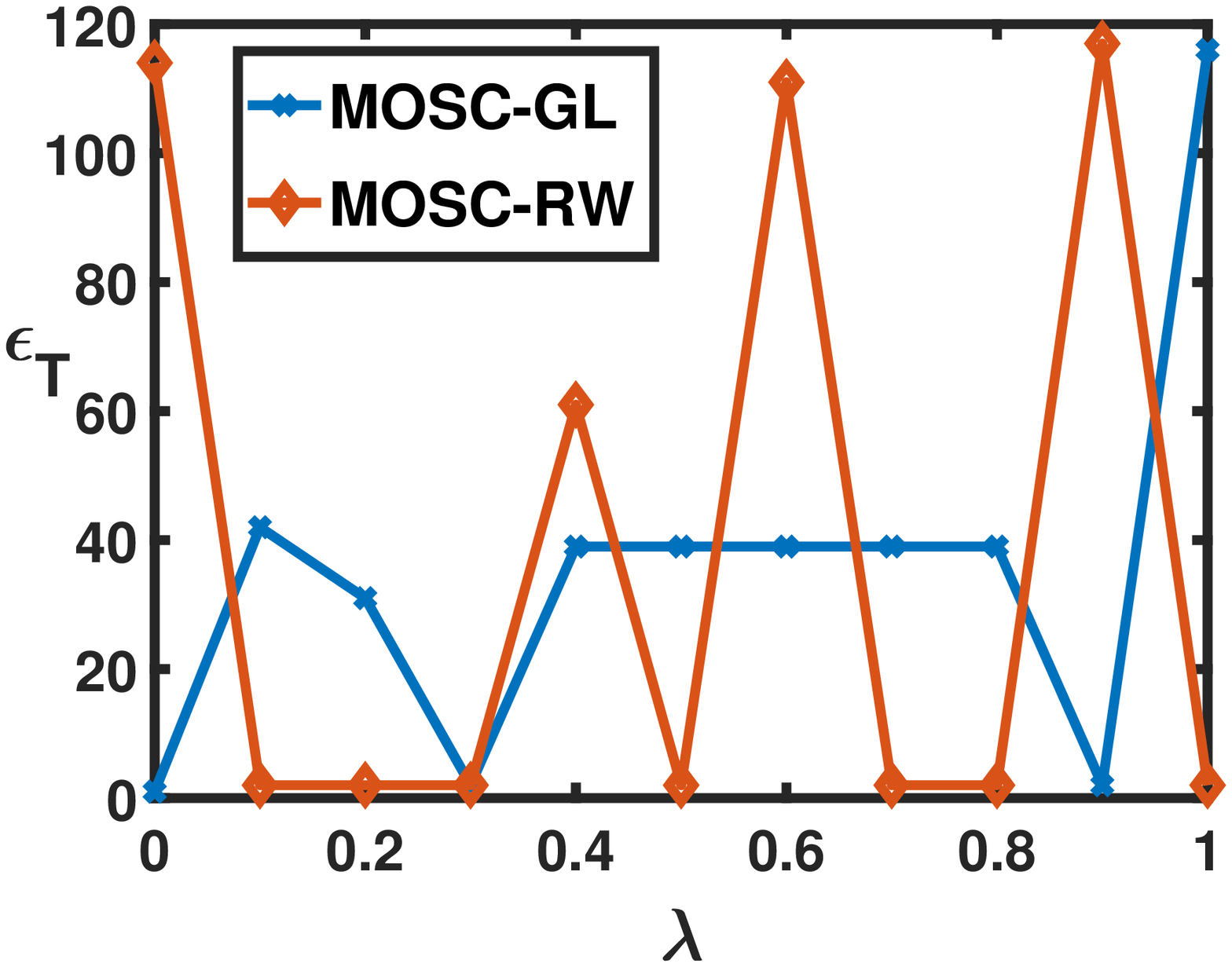}\label{senfig:subfig9}}
\caption{Sensitivity analysis of $\lambda$ on DBLP, Orkut and Football w.r.t NMI, $\epsilon_N$ and $\epsilon_T$.}\label{sen_multi_cut_1}
\end{figure}

\subsection{Effect of Cut Criteria} \label{cut_cmp}
Firstly, we study the effect of cut criteria on clustering performance. Note that Nassoc$_2$ ($\nassocTwo$) and Ncut$_2$ ($\ntwo$) are equivalent~\cite{shi2000normalized},  but Nassoc$_3$ ($\nassoc$) and Ncut$_3$ ($\nth$) are not (see Appendix). Thus, we have seven different cut criteria 
from Table \ref{table:cmp_cut} and the proposed mixed-order conductance ($\phi_X$) defined in Eq. (\ref{mix-conductance}). We study their effect on all algorithms while
only reporting representative results from SC-Shi, HOSVD and MOSC-RW ($\lambda=0.5$) on YouTube.

 Fig. \ref{multi_cut} shows the results for the two metrics $\epsilon_{N}$ and 
$\epsilon_{T}$, together with the Ocut values (the best possible cut, knowing the ground truth). We have the following observations:

\begin{enumerate}
\item Cut criteria have significant impact on the performance. Among all criteria, Nassoc$_3$ consistently gives the smallest $\epsilon_{T}$ for all algorithms on YouTube. 
\item The metric $\epsilon_{N}$ does not have a \textbf{monotonic} relationship with $\epsilon_{T}$.
We show an example on YouTube in Fig. \ref{multi_cut}\subref{cutfig:subfig4}
with respect to $\epsilon_{N}$ and $\epsilon_{T}$. As we can see, 
HOSVD+Ncut$_3$ vs. SC-Shi+Conduct$_2$ on $\epsilon_{N}$  has
opposite relationship on $\epsilon_{T}$.
\item The Ocut of MOSC-RW outperforms those of SC-Shi and HOSVD on both metrics $\epsilon_{N}$ and $\epsilon_{E}$ in YouTube, indicating the greater potential of MOSC-RW (even with a fixed $\lambda$). Regarding Ocut of MOSC-GL, we will present the results in Sec. \ref{sec:performance}.
\item Nassoc$_3$ and Ncut$_3$ are indeed not equivalent, as to be proven in Appendix.
\end{enumerate}

\renewcommand{\arraystretch}{1.5}
\begin{table*}[t]
	\caption{Performance of clustering algorithms with the best cut criteria. The best is in \textbf{bold} and the second best is \underline{underlined}, not considering \TOOLNAME(O-$\lambda$). A larger NMI indicates a better result, while a smaller $\epsilon_N$/$\epsilon_E$/$\epsilon_T$ indicates a better result. The corresponding best cut criteria are shown in parenthesis: Conduct$_3$ ($\phi_3$), Nassoc$_3$ ($\cncut_3$), Expan$_3$ ($\alpha_3$), Ncut$_3$ ($\nth$), Con-X ($\phi_X$),
		Conduct$_2$ ($\phi_2$), Expan$_2$ ($\alpha_2$), Nassoc$_2$ ($\cncut_2$), $k$-means(KM).}
	\resizebox{\textwidth}{!}{
	\Huge
		\centering \label{bestCmp}
		\scalebox{0.6}{
			\begin{tabular}{|c|c|cc|ccccc|ccc|ccc|}
				\hline
            	\multicolumn{2}{|c|}{} & \multicolumn{2}{c|}{Second order} & \multicolumn{5}{c|}{Third order} &\multicolumn{3}{c|}{MOSC-RW}&\multicolumn{3}{c|}{MOSC-GL}\\
				\hline
				\multicolumn{2}{|c|}{Method} & SC-Shi & SC-Ng & HOSVD & MSC & TSC &HOSPLOC & STSC &$\lambda$ = 0.5 & Auto-$\lambda$& O-$\lambda$&$\lambda$ = 0.5  & Auto-$\lambda$& O-$\lambda$\\
				\toprule 
				\multirow{3}{*}{\rotatebox{90}{DBLP}} 
				&  NMI &0.650 ($\phi_2$)	&\textbf{0.656} (KM)	&0.550 ($\xi_2$)	&0.620 ($\phi_2$)	&0.648 ($\phi_2$)	&0.286 ($\phi_3$)	&0.628 ($\phi_2$)	&\underline{0.654} ($\phi_2$)	&0.646 ($\phi_2$)	&0.795 (KM)	&0.648 ($\phi_2$)	&0.645 ($\phi_2$)	&0.711 (KM)\\ %
				& \EN &\textbf{4.13} ($\eta_3$)	&4.56 ($\eta_3$)	&5.40 ($\xi_3$)	&4.65 (KM)	&4.30 (KM)	&17.28 ($\phi_3$)	&\underline{4.24} (KM)	&4.43 ($\eta_3$)	&4.76 ($\phi_X$)	&1.76 (KM)	&4.82 (KM)	&4.87 ($\phi_2$)	&2.76 (KM)\\
				& \EE   &\textbf{13.36} ($\xi_3$)	&14.58 ($\xi_3$)	&17.99 ($\xi_3$)	&15.98 ($\xi_3$)	&18.72 ($\xi_3$)	&70.49 ($\phi_3$)	&18.72 ($\xi_3$)	&\underline{14.27} ($\xi_3$)	&15.68 ($\xi_3$)	&6.67 (KM)	&15.87 ($\xi_3$)	&16.67 ($\xi_2$)	&9.56 (KM)		\\
				& \ET &\textbf{24.81} ($\xi_3$)	&\underline{26.84} ($\xi_3$)	&32.18 ($\xi_3$)	&29.57 ($\xi_2$)	&37.93 ($\xi_3$)	&236.65 ($\phi_3$)	&45.46 ($\xi_3$)	&29.23 ($\xi_3$)	&28.46 ($\xi_3$)	&11.85 (KM)	&30.30 ($\xi_3$)	&32.01 ($\eta_3$)	&20.37 (KM)	\\
				\hline
				\multirow{3}{*}{\rotatebox{90}{YouTube}} 
				& NMI &0.248 ($\eta_3$)	&0.270 ($\eta_3$)	&0.124 ($\xi_2$)	&0.184 (KM)	&-	&-	&0.150 ($\xi_2$)	&\textbf{0.284} (KM)	&0.260 ($\phi_2$)	&0.386 (KM)	&\underline{0.275} (KM)	&0.263 ($\phi_2$)	&0.335 ($\eta_3$)\\ %
				& \EN  &\underline{22.48} ($\xi_2$)	&23.31 ($\xi_2$)	&25.69 ($\alpha_3$)	&24.41 ($\alpha_3$)	&-	&-	&23.83 ($\xi_3$)	&\textbf{22.18} ($\xi_2$)	&23.46 ($\xi_2$)	&16.38 (KM)	&23.44 ($\xi_2$)	&23.83 ($\xi_2$)	&18.45 ($\xi_3$)	\\
				& \EE&\underline{44.42} ($\xi_2$)	&46.46 ($\alpha_2$)	&50.61 ($\alpha_3$)	&47.18 ($\alpha_2$)	&-	&-	&63.12 ($\xi_3$)	&\textbf{44.28} ($\xi_2$)	&46.74 ($\alpha_2$)	&36.08 ($\xi_2$)	&52.58 ($\xi_2$)	&47.96 ($\alpha_2$)	&37.96 (KM)	\\
				& \ET   &\textbf{27.70} ($\xi_3$)	&29.49 ($\alpha_2$)	&30.78 ($\xi_3$)	&29.1 ($\alpha_2$)	&-	&-	&59.78 ($\xi_3$)	&\underline{28.90} ($\alpha_2$)	&29.29 ($\alpha_2$)	&22.15 ($\xi_3$)	&38.63 ($\xi_3$)	&29.46 ($\alpha_2$)	&22.95 (KM)	\\
				\hline
				\multirow{3}{*}{\rotatebox{90}{Orkut}} 
				&  NMI &\underline{0.397} ($\eta_3$)	&0.397 ($\eta_3$)	&0.3618 ($\xi_2$)	&0.390 ($\xi_2$)	&-	&-	&0.387 ($\xi_2$)	&\textbf{0.410} (KM)	&0.393 ($\xi_2$)	&0.472 (KM)	&0.397 (KM)	&0.394 ($\xi_2$)	&0.423 ($\eta_3$)	\\ %
				& \EN &37.13 ($\xi_2$)	&37.09 ($\xi_2$)	&40.43 ($\xi_2$)	&38.49 ($\eta_3$)	&-	&-	&37.60 ($\xi_2$)	&\textbf{36.05} (KM)	&37.02 ($\xi_2$)	&28.24 (KM)	&\underline{36.72} ($\xi_2$)	&36.93 ($\xi_2$)	&32.98 (KM)	\\
				& \EE &574.6 ($\xi_3$)	&574.6 ($\xi_3$)	&624.7 ($\xi_3$)	&582.3 ($\xi_3$)	&-	&-	&569.2 ($\xi_3$)	&\textbf{521.6} ($\alpha_2$)	&571.8 ($\xi_3$)	&410.3 ($\alpha_3$)	&\underline{550.4} ($\xi_3$)	&574.9 ($\xi_3$)	&497.7 ($\xi_3$)	\\
				& \ET  &4557 ($\xi_3$)	&4557 ($\xi_3$)	&4937 ($\xi_3$)	&4575 ($\xi_3$)	&-	&-	&5104 ($\xi_3$)	&\textbf{3949} ($\alpha_3$)	&4541 ($\xi_3$)	&2944 ($\alpha_3$)	&\underline{4405} ($\xi_3$)	&4614 ($\xi_3$)	&3816 ($\xi_3$)\\
				\hline
				\multirow{3}{*}{\rotatebox{90}{LJ}} 
				&  NMI &\underline{0.226} ($\xi_3$)	&0.224 ($\xi_3$)	&0.218 ($\eta_3$)	&0.224 ($\xi_3$)	&0.214 ($\xi_2$)	&-	&0.201 ($\eta_3$)	&\textbf{0.229} ($\xi_3$)	&0.221 ($\xi_3$)	&0.271 (KM)	&0.208 ($\xi_3$)	&0.212 ($\xi_3$)	&0.238 ($\xi_3$)	\\ %
								& \EN &5.58 ($\xi_3$)	&5.63 ($\xi_3$)	&5.79 ($\xi_3$)	&5.74 ($\xi_3$)	&5.52 ($\xi_3$)	&-	&\textbf{5.15} ($\xi_3$)	&\underline{5.49} ($\xi_3$)	&5.66 ($\xi_3$)	&4.43 ($\xi_3$)	&5.76 ($\xi_3$)	&5.64 ($\xi_3$)	&4.98 ($\xi_3$)	\\
				& \EE   &\underline{49.83} ($\xi_3$)	&50.01 ($\xi_3$)	&55.09 ($\xi_3$)	&52.54 ($\xi_3$)	&58.19 ($\xi_3$)	&-	&57.06 ($\xi_3$)	&\textbf{47.88} ($\xi_3$)	&51.01 ($\xi_3$)	&41.63 ($\xi_3$)	&58.64 ($\xi_3$)	&54.17 ($\xi_3$)	&46.51 ($\xi_3$)	\\
				& \ET  &\underline{546.1} ($\xi_3$)	&547.6 ($\xi_3$)	&600.6 ($\xi_3$)	&574.5 ($\xi_3$)	&737.7 ($\xi_3$)	&-	&773.0 ($\xi_3$)	&\textbf{530.6} ($\xi_3$)	&556.4 ($\xi_3$)	&460.2 ($\xi_3$)	&730.3 ($\xi_3$)	&617.8 ($\xi_3$)	&539.2 ($\xi_3$)	\\
				\bottomrule
	\end{tabular}}}
\end{table*}

\renewcommand{\arraystretch}{1}
\begin{table*}[h]
	\caption{Clustering performance of algorithms with the best cut criteria. The best is in \textbf{bold} and the second 
	best is \underline{underlined}, not considering \TOOLNAME(O-$\lambda$), which makes use of the ground truth. A larger NMI indicates a better result, while a smaller $\epsilon_N$/$\epsilon_E$/$\epsilon_T$ indicates a better result. Note that there are ties.}
	\resizebox{\textwidth}{!}{
		\centering \label{fullCmp}
		\scalebox{0.6}{
			\begin{tabular}{|c|c|cc|cccc|ccc|ccc|}
				\hline
            	\multicolumn{2}{|c|}{} & \multicolumn{2}{c|}{Second order} & \multicolumn{4}{c|}{Third order} &\multicolumn{3}{c|}{MOSC-RW}&\multicolumn{3}{c|}{MOSC-GL}\\
				\hline
				\multicolumn{2}{|c|}{Method} & SC-Shi & SC-Ng & HOSVD & MSC & TSC & STSC &$\lambda$ = 0.5 & Auto-$\lambda$& O-$\lambda$&$\lambda$ = 0.5  & Auto-$\lambda$& O-$\lambda$\\
				\hline
				\multirow{3}{*}{\rotatebox{90}{Zachary}} 

				& NMI &\textbf{0.837}	&\textbf{0.837}	&0.069	&\underline{0.732}	&0.677	&0.325	&\textbf{0.837}	&\textbf{0.837}	&0.837	&\textbf{0.837}	&\textbf{0.837}	&0.837\\ %
				& \EN  &\textbf{1}	&\textbf{1}	&14	&\underline{2}	&\underline{2}	&8	&\textbf{1}	&\textbf{1}	&1	&\textbf{1}	&\textbf{1}	&1	\\
				& \EE&\textbf{2}	&\textbf{2}	&34	&\underline{3}	&\underline{3}	&24	&\textbf{2}	&\textbf{2}	&2	&\textbf{2}	&\textbf{2}	&2	\\
				& \ET   &\textbf{1}	&\underline{1}	&\underline{16}	&\textbf{1}	&\textbf{1}	&14	&\textbf{1}	&\textbf{1}	&1	&\textbf{1}	&\textbf{1}	&1	\\
				\hline
				\multirow{3}{*}{\rotatebox{90}{Football}} 

				&  NMI &0.883	&0.904	&0.896	&0.924	&0.866		&0.862	&0.924	&\underline{0.924}	&0.924	&0.9	&\textbf{0.931}	&0.938	\\ %
				& \EN &23	&15	&16	&\underline{10}	&26		&26	&10	&10	&10	&15	&\textbf{9}	&7	\\
				& \EE &63	&37	&\underline{36}	&\textbf{7}	&70		&72	&\textbf{7}	&\textbf{7}	&7	&\underline{36}	&\textbf{7}	&3	\\
				& \ET  &99	&50	&\underline{39}	&\textbf{2}	&110		&114	&\textbf{2}	&\textbf{2}	&2	&\underline{39}	&\textbf{2}	&1\\
				\hline
				\multirow{3}{*}{\rotatebox{90}{Polbooks}} 

				&  NMI &0.575	&0.542	&0.092	&0.542	&0.180	&0.103	&\underline{0.575}	&0.575	&0.581	&0.563	&\textbf{0.589}	&0.589	\\ %
				& \EN &\textbf{17}	&\underline{18}	&56	&\underline{18}	&55		&51	&\textbf{17}	&\textbf{17}	&16	&\textbf{17}	&\textbf{17}	&17	\\
				& \EE  &\underline{27}	&33	&185	&34	&281	&172	&\underline{27}	&\underline{27}	&27	&28	&\textbf{21}	&21	\\
				& \ET   &\underline{7}	&10	&234	&8	&384	&227	&\underline{7}	&\underline{7}	&7	&\underline{7}	&\textbf{1}	&1	\\
				\hline
								\multirow{3}{*}{\rotatebox{90}{Dolphin}} 

				&  NMI &\underline{0.889}	&\underline{0.889}	&0.081	&0.536	&0.582		&0.631	&\underline{0.889}	&\underline{0.889}	&\underline{0.889}	&\underline{0.889}	&\textbf{1}	&1	\\ %
				& \EN &\underline{1}	&\underline{1}	&19	&7	&6		&5	&\underline{1}	&\underline{1}	&\underline{1}	&\underline{1}	&\textbf{0}	&0	\\
				& \EE  &\underline{1}	&\underline{1}	&43	&10	&8		&6	&\underline{1}	&\underline{1}	&\underline{1}	&1	&\textbf{0}	&0	\\
				& \ET   &\textbf{0}	&\textbf{0}	&29	&\textbf{0}	&\underline{1}		&\textbf{0}	&\textbf{0}	&0	&0	&\textbf{0}	&\textbf{0}	&0	\\
				\hline
								\multirow{3}{*}{\rotatebox{90}{PBlogs}} 

				&  NMI &0.007	&0.007	&0.014	&0.023	&-		&\underline{0.430}	&0.012	&\textbf{0.458}	&0.557	&0.098	&0.016	&0.098\\ %
				& \EN &671	&732	&677	&614	&-		&\textbf{204}	&659	&\underline{230}	&151	&478	&647	&478\\
				& \EE   &7302	&7302	&7307	&7260	&-		&\underline{362}	&7302	&\textbf{184}	&155	&7302	&7301	&7301		\\
				& \ET &36401	&36402	&36400	&36400	&-		&\underline{631}	&36401.5	&\textbf{456}	&410	&36402	&36400	&36400\\
				\hline
	\end{tabular}}}\vspace{-3mm}
\end{table*}

\subsection{Effect of $\lambda$} 
The mixing parameter $\lambda$ is the only hyperparamter in MOSC. 
To gain insight of MOSC, we conduct sensitivity analysis on $\lambda$ as shown in
Fig. \ref{sen_multi_cut_1} w.r.t. NMI, $\epsilon_N$ and $\epsilon_T$. We have the following
observations:

\begin{enumerate}
\item The choice of $\lambda$ can significantly affect the performance in all datasets. This was the motivation of developing schemes to  automatically determine the best $\lambda$.
\item Overall, MOSC-RW is more sensitive to $\lambda$ than MOSC-GL.
We will explain the reason at the end of Sec. \ref{sec:performance}.
\end{enumerate}

\subsection{Performance Comparison} \label{sec:performance}
We study the results of all algorithms in combination of all eight criteria and $k$-means (KM).
Figure \ref{multi_cut} shows that cut criteria can affect the performance of all algorithms. Therefore,
for fair comparison, we report the clustering results conducted by the  best criteria 
for each algorithm. The top two results are in bold (best) or underlined (second best), without considering the O-$\lambda$ version since it uses the ground truth. O-$\lambda$ is only used for gaining insights rather than direct performance comparison with existing methods.

\textbf{Results on SNAP Networks.} We show the performance of all 
clustering algorithms with the best cut criteria in terms of NMI, $\epsilon_N$, $\epsilon_E$, and $\epsilon_T$ on SNAP networks in Table \ref{bestCmp}. In addition, we report 
the best cut criteria.
The results for some settings of TSC and HOSPLOC are not available either due to long running time (not finished within 40 hours) or out of memory. In particular, the multilinear PageRank algorithm in TSC is very expensive. 

We have four observations:
\begin{enumerate}
\item  MOSC-RW ($\lambda = 0.5$) achieves the best in 10 out of 16 settings; 
\item  MOSC-RW outperforms MOSC-GL, although MOSC-GL achieves top two results in 4 settings.
\item  Both MOSC-RW($\lambda = 0.5$) and MOSC-GL($\lambda = 0.5$) have better results than
MOSC-RW(Auto-$\lambda$) and MOSC-GL(Auto-$\lambda$). This demonstrates that a fixed 
mixing parameter is effective, but it also shows the automatic schemes are not effective
in these settings.
\item  MOSC-RW(O-$\lambda$) has the highest NMI and lowest $\epsilon_N$, $\epsilon_E$, and $\epsilon_T$ in all 16 settings. This indicates potential future improvement by designing a better strategy to determine $\lambda$.
\item  Third-order cut criteria outperform second-order cut criteria on the whole.
In all 16 settings, third-order cut criteria achieve the best performance in 8 settings among which
Nassoc$_3$ is the best in 6 settings. Second-order cut criteria achieves the best in 5 settings and $k$-means for 3 settings. 

\end{enumerate}

\begin{figure*} [t]
	\centering 
	\subfigure[Dolphin]{\includegraphics[width=77mm]{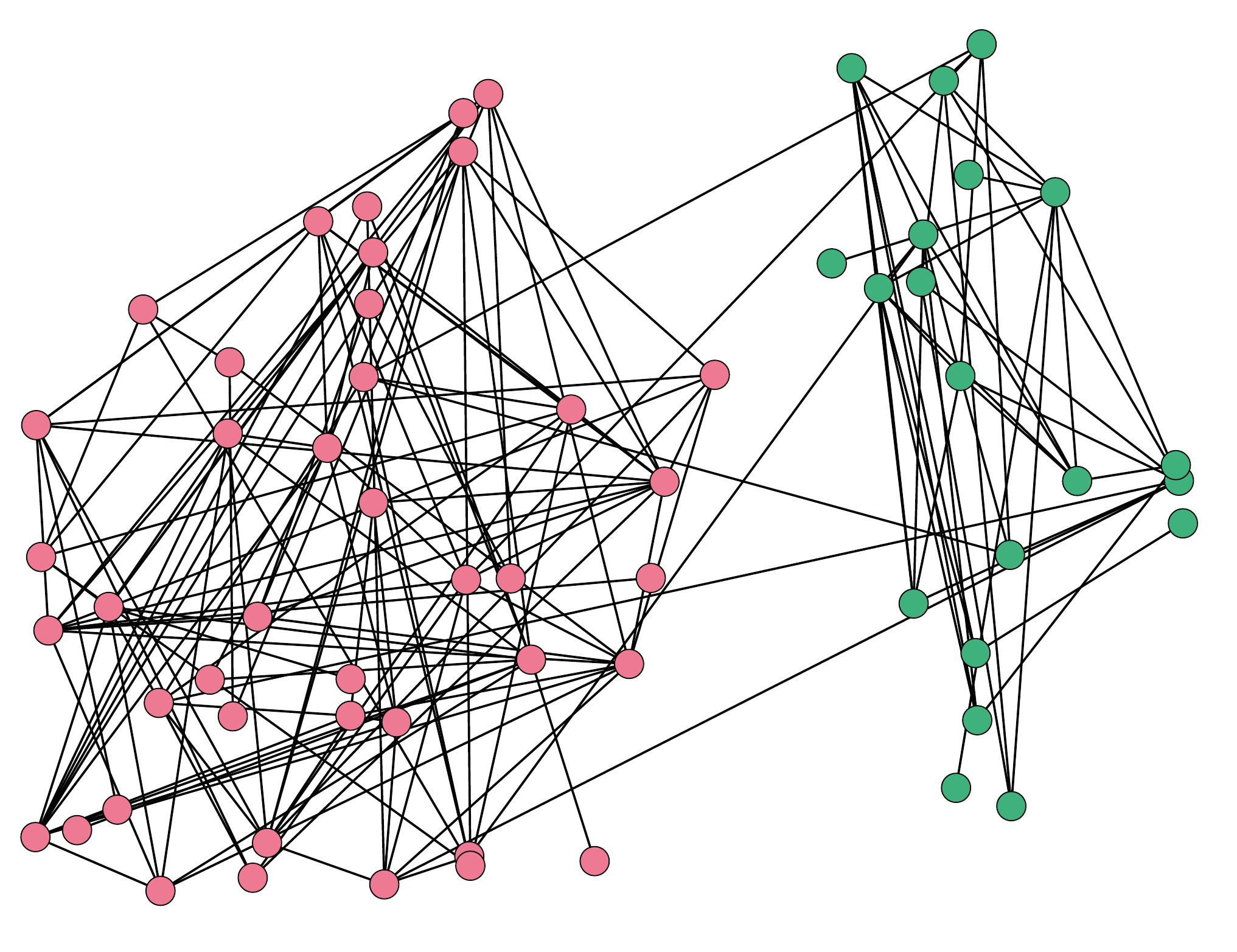}\label{dolphin:subfig1}}
	\subfigure[Football (12 clusters)]{\includegraphics[width=77mm]{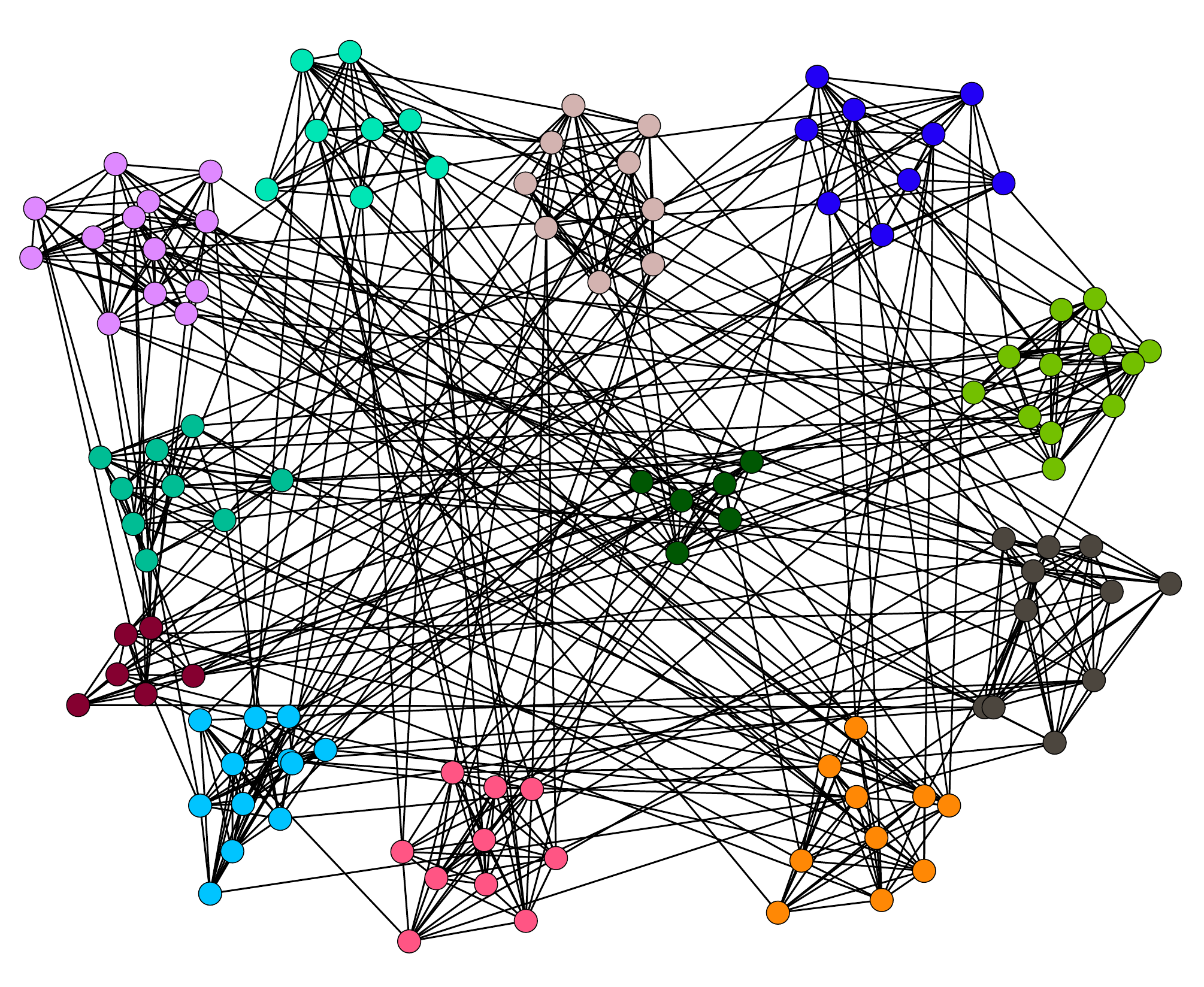}\label{football:subfig2}}
	\caption{Clusters in Dolphin and Football networks discovered by MOSC-GL (Auto-$\lambda$).}  \label{visualize}
\end{figure*}

\textbf{Results on Full Networks.}
We show the performance of all clustering algorithms with the best cut criteria for five full networks in terms of NMI, $\epsilon_N$, $\epsilon_E$, and $\epsilon_T$ in Table~\ref{fullCmp}, except HOSPLOC,
for which we were not able to obtain comparable results.
We have three observations:
\begin{enumerate}
\item MOSC-GL (Auto-$\lambda$) achieves the best performance in 16 out of 20 settings, demonstrating that automatic determination of $\lambda$ is effective in these settings. Specifically, for Dolphin,
MOSC-GL (Auto-$\lambda$) produces perfect results in all metrics. We visualise the obtained
clusters in Dolphin as shown in Fig. \ref{visualize}\subref{dolphin:subfig1}.
For networks with
multiple clusters (Polbooks, Football), MOSC-GL (Auto-$\lambda$) is also superior to others. We visualise Football with the 12 clusters in Fig. \ref{visualize}\subref{football:subfig2}.
\item MOSC-GL (Auto-$\lambda$) outperforms MOSC-RW (Auto-$\lambda$),
although MOSC-RW (Auto-$\lambda$) achieves the best results in 10 settings, which is still better than all existing SC algorithms (Note that there are ties).
\item MOSC-GL has better model quality than MOSC-RW based on Ocut. 
MOSC-GL (O-$\lambda$) achieves the highest NMI and lowest $\epsilon_N$, $\epsilon_E$, and $\epsilon_T$ in 15 out of 20 settings while
MOSC-RW (O-$\lambda$) achieves the best in 10 settings (again there are ties).
\end{enumerate}

From the above results, we see that MOSC-GL and MOSC-RW have different performance on
networks with different triangle densities. MOSC-RW tends to be better for networks with 
high triangle densities while MOSC-GL tends to be better for networks with 
low triangle densities. For MOSC-GL, $\matrics{W}_T$ can 
dominate $\matrics{W}_X$ in $\matrics{W}_X=(1-\lambda) \matrics{W}_T + \lambda \matrics{W}$,
especially for dense networks. Each entry of $\matrics{W}_T$ denotes the number of
triangles containing the corresponding edge while $\matrics{W}$ is a binary matrix.
Therefore, for most non-zero pairs $(i,j$), $\matrics{W}_T(i,j)$ is much larger than
$\matrics{W}(i,j)$ especially for dense networks. This can be the reason that MOSC-GL is less sensitive to tuning
$\lambda$, or finding the appropriate $\lambda$ is more difficult. That is, $\matrics{W}_X$
tends to encode much less edge information. In contrast, MOSC-RW does not have such
issue since $\matrics{A}$ and $\matrics{P}$ are normalised and thus they are in similar scales
before linear combination. Therefore, 
 MOSC-RW has a better performance than MOSC-GL in SNAP networks and the dense full graph PBlogs.

\subsection{Computational Time} 
Figure \ref{time} compares the computational time of different methods on YouTube,
LJ, PBlogs and Football, using $k$-means to obtain the final clusters to avoid the effect of cut criteria. We have the following observations:
\begin{enumerate}
\item Both HOSVD and MOSC-RW  involve tensor construction and operations so they  are both more time consuming, in particular on dense networks such as LJ and PBlogs, where HOSVD is the slowest and MOSC-RW is the second slowest. The reason is that HOSVD uses more a complicated dimension reduction method than
MOSC-RW.
\item MOSC-GL is more efficient than MOSC-RW in all cases and has similar efficiency
as conventional SC methods on the whole.
\end{enumerate}

\begin{figure}
    \centering
    \includegraphics[scale=0.5]{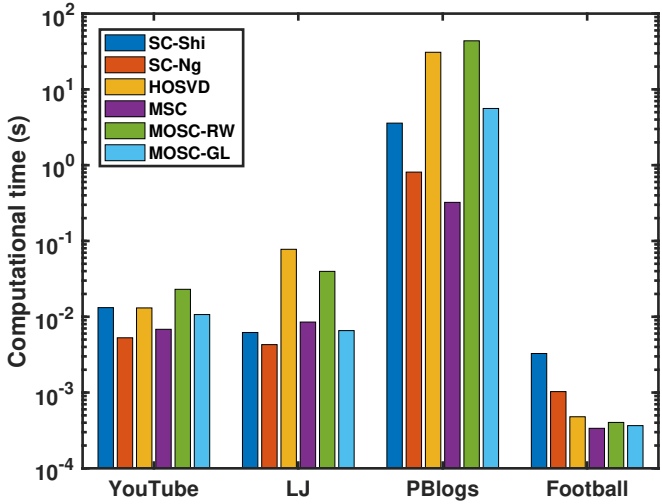}
    \caption{Computational time (in log scale) on YouTube, LJ, PBlogs and Football.} 
    \label{time}
\end{figure}

%% file: Conclusion.tex
\section{Conclusion}
This paper proposed two mixed-order spectral clustering (MOSC) methods, MOSC-GL and MOSC-RW, which model both second-order and third-order structures for network clustering. MOSC-GL combines edge and triangle adjacency matrices with theoretical performance guarantee. MOSC-RW combines first-order and second-order random walks with a probabilistic interpretation. We designed a scheme to automatically determine the mixing  parameter and new structure-aware error metrics for structure evaluation. We also defined an optimal cut to study the lower error bound for exploring potentials of MOSC. Experiments on 2,005 networks showed that MOSC algorithms outperform existing SC methods in most cases and the proposed mixed-order approach produces more superior clustering of networks.

%% file: Appendix.tex
\appendices
\section{Proof of the Relationship between Ncut$_3$ and Nassoc$_3$}\label{secApp}

Shi and Malik \cite{meila2001learning} showed that second-order normalised cut (Ncut$_2$) and association (Nassoc$_2$) are equivalent by the following equation: 
\begin{equation}
    Ncut_2 (S;G) = 2 - Nassoc_2 (S;G),
\end{equation}
which indicates that the minimisation of Ncut$_2$ is equivalent to maximisation of Nassoc$_2$.
It inspires us to think about a question: does third-order normalised cut (Ncut$_3$) \cite{li2017inhomogoenous} and association (Nassoc$_3$) \cite{ghoshdastidar2017uniform} have the similar  equivalence? 

Based on the definition of $cut_3(S)$, we have the following equation \cite{tsourakakis2017scalable} (omitting $G$): 

\begin{eqnarray} \label{cut_3}
cut_3(\grasubset{S},\bar{\grasubset{S}}) &=& \frac{1}{3} (((vol_3(\grasubset{S}) - assoc_3(\grasubset{S}))\nonumber\\
&&+(vol_3(\bar{\grasubset{S}}) - assoc_3(\bar{\grasubset{S}}))).
\end{eqnarray}
Ncut$_3 (S)$ \cite{li2017inhomogoenous} is defined as follows:
\begin{equation}
\begin{aligned}
Ncut_3 (\grasubset{S}) &=
cut_3({\grasubset{S},\bar{\grasubset{S}}})(\frac{1}{{vol_3({\grasubset{S}})}}+
\frac{1}{vol_3({\bar{\grasubset{S}})}}).\\&
\end{aligned}
\end{equation}
This gives that: 
\begin{equation}
\begin{aligned}
Ncut_3 (\grasubset{S}) &=\frac{(vol_3(\grasubset{S}) - assoc_3(\grasubset{S}))
+(vol_3(\bar{\grasubset{S}}) - assoc_3(\bar{\grasubset{S}}))}{{3 vol_3({\grasubset{S}})}}
\\&+\frac{(vol_3(\grasubset{S}) - assoc_3(\grasubset{S}))
+(vol_3(\bar{\grasubset{S}}) - assoc_3(\bar{\grasubset{S}}))}{{3 vol_3({\bar{\grasubset{S}}})}}
\\&=\frac{2}{3} - \frac{1}{3}(\frac{assoc_3(\grasubset{S})}{vol_3(\grasubset{S})}+
\frac{assoc_3(\bar{\grasubset{S}})}{vol_3(\bar{\grasubset{S}})})
\\&+\frac{1}{3} (\frac{{vol_3({\bar{\grasubset{S}}})}-assoc_3(\bar{{\grasubset{S}})}}{vol_3({\grasubset{S}})}+\frac{{vol_3({\grasubset{S}})}-assoc_3({\grasubset{S})}}{vol_3({\bar{\grasubset{S}}})})
\end{aligned}
\end{equation}

The definition of $Nassoc_3(S)$ \cite{ghoshdastidar2017uniform} is given as:
\begin{equation}
\begin{aligned}
Nassoc_3(S) = \frac{assoc_3(\grasubset{S})}{vol_3(\grasubset{S})}+
\frac{assoc_3(\bar{\grasubset{S}})}{vol_3(\bar{\grasubset{S}})}
\end{aligned}
\end{equation}

The above equations indicate:

\begin{equation}
\begin{aligned}
Ncut_3 (\grasubset{S})&= \frac{2}{3} - \frac{1}{3}Nassoc_3(S)\\&
+ \frac{1}{3} (\frac{{vol_3({\bar{\grasubset{S}}})}-assoc_3({\bar{\grasubset{S}})}}{vol_3({\grasubset{S}})}+\frac{{vol_3({\grasubset{S}})}
-assoc_3({\grasubset{S})}}{vol_3({\bar{\grasubset{S}}})}).\label{eq:relation}
\end{aligned}
\end{equation}
From Eq. (\ref{eq:relation}), we can see that the third term in is not a constant. Therefore,
minimisation  of Ncut$_3$ is not equivalent to maximisation of Nassoc$_3$.

%% file: main.bbl
\begin{thebibliography}{10}
\providecommand{\url}[1]{#1}
\csname url@samestyle\endcsname
\providecommand{\newblock}{\relax}
\providecommand{\bibinfo}[2]{#2}
\providecommand{\BIBentrySTDinterwordspacing}{\spaceskip=0pt\relax}
\providecommand{\BIBentryALTinterwordstretchfactor}{4}
\providecommand{\BIBentryALTinterwordspacing}{\spaceskip=\fontdimen2\font plus
\BIBentryALTinterwordstretchfactor\fontdimen3\font minus
  \fontdimen4\font\relax}
\providecommand{\BIBforeignlanguage}[2]{{%
\expandafter\ifx\csname l@#1\endcsname\relax
\typeout{** WARNING: IEEEtran.bst: No hyphenation pattern has been}%
\typeout{** loaded for the language `#1'. Using the pattern for}%
\typeout{** the default language instead.}%
\else
\language=\csname l@#1\endcsname
\fi
#2}}
\providecommand{\BIBdecl}{\relax}
\BIBdecl

\bibitem{benson2016higher}
A.~R. Benson, D.~F. Gleich, and J.~Leskovec, ``Higher-order organization of
  complex networks,'' \emph{Science}, vol. 353, no. 6295, pp. 163--166, 2016.

\bibitem{liu2014multiobjective}
C.~Liu, J.~Liu, and Z.~Jiang, ``A multiobjective evolutionary algorithm based
  on similarity for community detection from signed social networks,''
  \emph{IEEE Trans. Cybern.}, vol.~44, no.~12, pp. 2274--2287, 2014.

\bibitem{yang2015unified}
L.~Yang, X.~Cao, D.~Jin, X.~Wang, and D.~Meng, ``A unified semi-supervised
  community detection framework using latent space graph regularization,''
  \emph{IEEE Trans. Cybern.}, vol.~45, no.~11, pp. 2585--2598, 2015.

\bibitem{chang2017community}
H.~Chang, Z.~Feng, and Z.~Ren, ``Community detection using dual-representation
  chemical reaction optimization,'' \emph{IEEE Trans. Cybern.}, vol.~47,
  no.~12, pp. 4328--4341, 2017.

\bibitem{zhang2017mixed}
L.~Zhang, H.~Pan, Y.~Su, X.~Zhang, and Y.~Niu, ``A mixed representation-based
  multiobjective evolutionary algorithm for overlapping community detection,''
  \emph{IEEE Trans. Cybern.}, vol.~47, no.~9, pp. 2703--2716, 2017.

\bibitem{li2018multiobjective}
Z.~Li, J.~Liu, and K.~Wu, ``A multiobjective evolutionary algorithm based on
  structural and attribute similarities for community detection in attributed
  networks,'' \emph{IEEE Trans. Cybern.}, vol.~48, no.~7, pp. 1963--1976, 2018.

\bibitem{clauset2004finding}
A.~Clauset, M.~E. Newman, and C.~Moore, ``Finding community structure in very
  large networks,'' \emph{Physical Review E}, vol.~70, no.~6, p. 066111, 2004.

\bibitem{schaeffer2007graph}
S.~E. Schaeffer, ``Graph clustering,'' \emph{Computer science review}, vol.~1,
  no.~1, pp. 27--64, 2007.

\bibitem{POM09:communities}
M.~A. Porter, J.-P. Onnela, and P.~J. Mucha, ``Communities in networks,''
  \emph{Notices of the AMS}, vol.~56, no.~9, pp. 1082--1097, 2009.

\bibitem{fortunato2010community}
S.~Fortunato, ``Community detection in graphs,'' \emph{Physics Reports}, vol.
  486, no. 3-5, pp. 75--174, 2010.

\bibitem{New12:communities}
M.~E. Newman, ``Communities, modules and large-scale structure in networks.''
  \emph{Nature Physics}, vol.~8, no.~1, 2012.

\bibitem{meila2001learning}
M.~Meila and J.~Shi, ``Learning segmentation by random walks,'' in
  \emph{Advances in Neural Information Processing Systems}, 2001, pp. 873--879.

\bibitem{ng2002spectral}
A.~Y. Ng, M.~I. Jordan, and Y.~Weiss, ``On spectral clustering: Analysis and an
  algorithm,'' in \emph{Advances in Neural Information Processing Systems},
  2002, pp. 849--856.

\bibitem{benson2015tensor}
A.~R. Benson, D.~F. Gleich, and J.~Leskovec, ``Tensor spectral clustering for
  partitioning higher-order network structures,'' in \emph{Proceedings of the
  2015 SIAM International Conference on Data Mining}.\hskip 1em plus 0.5em
  minus 0.4em\relax SIAM, 2015, pp. 118--126.

\bibitem{zhou2017local}
D.~Zhou, S.~Zhang, M.~Y. Yildirim, S.~Alcorn, H.~Tong, H.~Davulcu, and J.~He,
  ``A local algorithm for structure-preserving graph cut,'' in
  \emph{Proceedings of the 23rd ACM SIGKDD International Conference on
  Knowledge Discovery and Data Mining}.\hskip 1em plus 0.5em minus 0.4em\relax
  ACM, 2017, pp. 655--664.

\bibitem{milo2002network}
R.~Milo, S.~Shen-Orr, S.~Itzkovitz, N.~Kashtan, D.~Chklovskii, and U.~Alon,
  ``Network motifs: simple building blocks of complex networks,''
  \emph{Science}, vol. 298, no. 5594, pp. 824--827, 2002.

\bibitem{sporns2004motifs}
O.~Sporns and R.~K{\"o}tter, ``Motifs in brain networks,'' \emph{PLoS Biology},
  vol.~2, no.~11, p. e369, 2004.

\bibitem{granovetter1977strength}
M.~S. Granovetter, ``The strength of weak ties,'' in \emph{Social
  Networks}.\hskip 1em plus 0.5em minus 0.4em\relax Elsevier, 1977, pp.
  347--367.

\bibitem{kossinets2006empirical}
G.~Kossinets and D.~J. Watts, ``Empirical analysis of an evolving social
  network,'' \emph{science}, vol. 311, no. 5757, pp. 88--90, 2006.

\bibitem{serrour2011detecting}
B.~Serrour, A.~Arenas, and S.~G{\'o}mez, ``Detecting communities of triangles
  in complex networks using spectral optimization,'' \emph{Computer
  Communications}, vol.~34, no.~5, pp. 629--634, 2011.

\bibitem{ghoshdastidar2015spectral}
D.~Ghoshdastidar and A.~Dukkipati, ``Spectral clustering using multilinear svd:
  Analysis, approximations and applications.'' in \emph{AAAI}, 2015, pp.
  2610--2616.

\bibitem{ghoshdastidar2014consistency}
------, ``Consistency of spectral partitioning of uniform hypergraphs under
  planted partition model,'' in \emph{Advances in Neural Information Processing
  Systems}, 2014, pp. 397--405.

\bibitem{ghoshdastidar2017uniform}
------, ``Uniform hypergraph partitioning: Provable tensor methods and sampling
  techniques,'' \emph{Journal of Machine Learning Research}, vol.~18, no.~50,
  pp. 1--41, 2017.

\bibitem{wu2016general}
T.~Wu, A.~R. Benson, and D.~F. Gleich, ``General tensor spectral co-clustering
  for higher-order data,'' in \emph{Advances in Neural Information Processing
  Systems}, 2016, pp. 2559--2567.

\bibitem{tsourakakis2017scalable}
C.~E. Tsourakakis, J.~Pachocki, and M.~Mitzenmacher, ``Scalable motif-aware
  graph clustering,'' in \emph{Proceedings of the 26th International Conference
  on World Wide Web}.\hskip 1em plus 0.5em minus 0.4em\relax International
  World Wide Web Conferences Steering Committee, 2017, pp. 1451--1460.

\bibitem{yin2017local}
H.~Yin, A.~R. Benson, J.~Leskovec, and D.~F. Gleich, ``Local higher-order graph
  clustering,'' in \emph{Proceedings of the 23rd ACM SIGKDD International
  Conference on Knowledge Discovery and Data Mining}.\hskip 1em plus 0.5em
  minus 0.4em\relax ACM, 2017, pp. 555--564.

\bibitem{von2007tutorial}
U.~Von~Luxburg, ``A tutorial on spectral clustering,'' \emph{Statistics and
  Computing}, vol.~17, no.~4, pp. 395--416, 2007.

\bibitem{xu2016representing}
J.~Xu, T.~L. Wickramarathne, and N.~V. Chawla, ``Representing higher-order
  dependencies in networks,'' \emph{Science Advances}, vol.~2, no.~5, p.
  e1600028, 2016.

\bibitem{malliaros2013clustering}
F.~D. Malliaros and M.~Vazirgiannis, ``Clustering and community detection in
  directed networks: A survey,'' \emph{Physics Reports}, vol. 533, no.~4, pp.
  95--142, 2013.

\bibitem{shi2000normalized}
J.~Shi and J.~Malik, ``Normalized cuts and image segmentation,'' \emph{IEEE
  Trans. on Pattern Analysis and Machine Intelligence}, vol.~22, no.~8, pp.
  888--905, 2000.

\bibitem{li2017inhomogoenous}
P.~Li and O.~Milenkovic, ``Inhomogoenous hypergraph clustering with
  applications,'' in \emph{Advances in Neural Information Processing Systems},
  2017, pp. 2305--2315.

\bibitem{flake2004graph}
G.~W. Flake, R.~E. Tarjan, and K.~Tsioutsiouliklis, ``Graph clustering and
  minimum cut trees,'' \emph{Internet Mathematics}, vol.~1, no.~4, pp.
  385--408, 2004.

\bibitem{fiedler1973algebraic}
M.~Fiedler, ``Algebraic connectivity of graphs,'' \emph{Czechoslovak
  Mathematical Journal}, vol.~23, no.~2, pp. 298--305, 1973.

\bibitem{lee2014multiway}
J.~R. Lee, S.~O. Gharan, and L.~Trevisan, ``Multiway spectral partitioning and
  higher-order cheeger inequalities,'' \emph{Journal of the ACM (JACM)},
  vol.~61, no.~6, p.~37, 2014.

\bibitem{leskovec2010empirical}
J.~Leskovec, K.~J. Lang, and M.~Mahoney, ``Empirical comparison of algorithms
  for network community detection,'' in \emph{Proceedings of the 19th
  International Conference on World Wide Web}.\hskip 1em plus 0.5em minus
  0.4em\relax ACM, 2010, pp. 631--640.

\bibitem{chakraborty2017metrics}
T.~Chakraborty, A.~Dalmia, A.~Mukherjee, and N.~Ganguly, ``Metrics for
  community analysis: A survey,'' \emph{ACM Computing Surveys (CSUR)}, vol.~50,
  no.~4, p.~54, 2017.

\bibitem{yang2015defining}
J.~Yang and J.~Leskovec, ``Defining and evaluating network communities based on
  ground-truth,'' \emph{Knowledge and Information Systems}, vol.~42, no.~1, pp.
  181--213, 2015.

\bibitem{ana2003robust}
L.~Ana and A.~K. Jain, ``Robust data clustering,'' in \emph{Computer Vision and
  Pattern Recognition, 2003. Proceedings. 2003 IEEE Computer Society Conference
  on}, vol.~2.\hskip 1em plus 0.5em minus 0.4em\relax IEEE, 2003, pp. II--II.

\bibitem{zachary1977information}
W.~W. Zachary, ``An information flow model for conflict and fission in small
  groups,'' \emph{Journal of Anthropological Research}, vol.~33, no.~4, pp.
  452--473, 1977.

\bibitem{lusseau2003emergent}
D.~Lusseau, ``The emergent properties of a dolphin social network,''
  \emph{Proceedings of the Royal Society of London B: Biological Sciences},
  vol. 270, no. Suppl 2, pp. S186--S188, 2003.

\bibitem{newman2006modularity}
M.~E. Newman, ``Modularity and community structure in networks,''
  \emph{Proceedings of the national academy of sciences}, vol. 103, no.~23, pp.
  8577--8582, 2006.

\bibitem{adamic2005political}
L.~A. Adamic and N.~Glance, ``The political blogosphere and the 2004 us
  election: divided they blog,'' in \emph{Proceedings of the 3rd international
  workshop on Link discovery}.\hskip 1em plus 0.5em minus 0.4em\relax ACM,
  2005, pp. 36--43.

\bibitem{gleich2015multilinear}
D.~F. Gleich, L.-H. Lim, and Y.~Yu, ``Multilinear pagerank,'' \emph{SIAM
  Journal on Matrix Analysis and Applications}, vol.~36, no.~4, pp. 1507--1541,
  2015.

\end{thebibliography}
